\documentclass[11pt]{article}
\pdfoutput=1
\usepackage{enumerate}
\usepackage{pdfsync}
\usepackage[OT1]{fontenc}
\usepackage[dvipsnames]{xcolor}

\usepackage[colorlinks,
            linkcolor=red,
            anchorcolor=blue,
            citecolor=blue
            ]{hyperref}
\usepackage{fullpage}
\usepackage[protrusion=true,expansion=true]{microtype}
\usepackage{pbox}
\usepackage{setspace}
\usepackage{tabularx}
\usepackage{float}
\usepackage{wrapfig,lipsum}
\usepackage{enumitem}
\usepackage{natbib}
\usepackage{multirow}

\makeatletter
\newcommand*{\rom}[1]{\expandafter\@slowromancap\romannumeral #1@}
\makeatother

\usepackage{dsfont}
\usepackage{smile}
            
\def \bs {\backslash}

\def \dd {\text{d}}
\newcommand{\la}{\langle}
\newcommand{\ra}{\rangle}

\def \transitSGLD {\cT}
\def \transitRef {\cT^{\star}}
\def \stationalRef {\pi^{\star}}

\allowdisplaybreaks

\begin{document}
\title{
\huge Faster Convergence  of Stochastic Gradient Langevin Dynamics for Non-Log-Concave Sampling
}
\author
{
	Difan Zou\thanks{Department of Computer Science, University of California, Los Angeles, Los Angeles, CA 90095; e-mail: {\tt knowzou@cs.ucla.edu}} 
	~~~and~~~
	Pan Xu\thanks{Department of Computer Science, University of California, Los Angeles, Los Angeles, CA 90095; e-mail: {\tt panxu@cs.ucla.edu}} 
    ~~~and~~~
    Quanquan Gu\thanks{Department of Computer Science, University of California, Los Angeles, Los Angeles, CA 90095; e-mail: {\tt qgu@cs.ucla.edu}}
}
\date{}
\maketitle
\begin{abstract}
We provide a new  convergence analysis of stochastic gradient Langevin dynamics (SGLD) for sampling from a class of distributions that can be non-log-concave.  At the core of our approach is a novel conductance analysis of SGLD using an auxiliary time-reversible Markov Chain. Under certain conditions on the target distribution, we prove that $\tilde O(d^4\epsilon^{-2})$ stochastic gradient evaluations suffice to guarantee $\epsilon$-sampling error in terms of the total variation distance, where $d$ is the problem dimension. This improves existing results on the convergence rate of SGLD \citep{raginsky2017non,xu2018global}.
We further show that provided an additional Hessian Lipschitz condition on the log-density function, SGLD is guaranteed to achieve $\epsilon$-sampling error within $\tilde O(d^{15/4}\epsilon^{-3/2})$ stochastic gradient evaluations.
Our proof technique provides a new way to study the convergence of Langevin-based algorithms and sheds some light on the design of fast stochastic gradient-based sampling algorithms.

\end{abstract}

\section{Introduction}
We study the problem of sampling from a target distribution using Langevin dynamics \citep{langevin1908theory} based algorithms. Mathematically,  Langevin dynamics (a.k.a., overdamped Langevin dynamics) is defined by the following stochastic differential equation (SDE)
\begin{align}\label{eq:langevin_dynamics}
\dd  \bX(t) = -\nabla f\big(\bX(t)\big) \dd t + \sqrt{2\beta^{-1}}\dd \bB(t),
\end{align}
where $\beta>0$ is called the inverse temperature parameter and $\bB(t)\in\RR^d$ is the Brownian motion at time $t$. It has been proved in \cite{chiang1987diffusion,roberts1996exponential} that under certain conditions on the drift term $-\nabla f(\bX(t))$, the Langevin dynamics will converge to a unique stationary distribution $\pi(\dd \xb)\propto e^{-\beta f(\xb)}\dd \xb$. To approximately sample from such a target distribution $\pi$, we can apply the Euler-Maruyama discretization onto \eqref{eq:langevin_dynamics}, leading to the Langevin Monte Carlo algorithm (LMC), which iteratively updates the parameter $\xb_k$ as follows
\begin{align}\label{eq:def_lmc}
\xb_{k+1} = \xb_k - \eta\nabla f(\xb_k) + \sqrt{2\eta\beta^{-1}}\cdot\bepsilon_k,
\end{align}
where $k=0,1,\ldots$ denotes the time step, $\{\bepsilon_k\}_{k=0,1,\ldots}$ are i.i.d. standard Gaussian random vectors in $\RR^d$, and $\eta>0$ is the step size of the discretization.

In large scale machine learning problems that involve a large amount of training data, the log-density function $f(\xb)$ can be typically formulated as the average of the log-density functions over all the training data points, i.e., $f(\xb) = n^{-1}\sum_{i=1}^n f_i(\xb)$\footnote{In some cases, the log-density function $f(\xb)$ is formulated as the sum of the log-density functions for training data points instead of the average. To cover these cases, we can simply transform the temperature parameter $\beta\rightarrow n\beta$ and thus the target distribution remains the same.
}, where $n$ is the size of training dataset and $f_i(\xb)$ denotes the log-density function for the $i$-th training data point. In these problems, the computation of the full gradient over the entire dataset can be very time-consuming. In order to save the cost of gradient computation, one can replace the full gradient $\nabla f(\xb)$ with a stochastic gradient computed only over a small subset of the dataset, which gives rise to stochastic gradient Langevin dynamics (SGLD) \citep{welling2011bayesian}. 

When the target distribution $\pi$ is log-concave, SGLD provably converges to $\pi$ at a sublinear rate in $2$-Wasserstein distance \citep{dalalyan2017user,dalalyan2017further,wang2019laplacian}. 
However, it becomes much more challenging to establish the convergence of SGLD when the target distribution is not log-concave. When the negative log-density function $f(\xb)$ is smooth and dissipative, the global convergence guarantee of SGLD has been firstly established in \citet{raginsky2017non}\footnote{Although this paper mainly focuses on the convergence analysis of SGLD for nonconvex optimization, part of its theoretical results also reveal the convergence rate for sampling from a target distribution.} via the optimal control theory and further improved in \citet{xu2018global}  by a direct analysis of the ergodicity of LMC. Nonetheless, these two works require extremely large mini-batch size (e.g., $B=\Omega(\epsilon^{-4})$) to ensure a sufficiently small sampling error, which is prohibitively large or even unrealistic compared with the practical setting. \citet{zhang2017hitting} studied the hitting time of SGLD for nonconvex optimization, but can only provide the convergence guarantee for finding a local minimum rather than converging to the target distribution. Recently, \citet{chau2019stochastic, zhang2019nonasymptotic} studied the global convergence of SGLD for nonconvex stochastic optimization problems and proved faster convergence rates than those in \citet{raginsky2017non,xu2018global}. However, their convergence results require an additional Lipschitz condition in terms of the input data (rather than the model parameter) on the stochastic gradients, which restricts their applications to a small class of SGLD-based sampling problems.

In this paper, we consider the same setting in \citet{raginsky2017non,xu2018global} and aim to establish faster convergence rates for SGLD with an arbitrary  mini-batch  size. 
In particular, we provide a new convergence analysis for SGLD based on an auxiliary time-reversible Markov chain called Metropolized SGLD \citep{zhang2017hitting}, which is constructed by adding a Metropolis-Hasting step to SGLD\footnote{This Markov chain is practically intractable and is only used for the sake of theoretical analysis.}. The key idea is that as long as the transition kernel of the constructed Metropolized SGLD chain is sufficiently close to that of SGLD, we can prove the convergence of SGLD to the target distribution. Compared with existing proof techniques that typically take LMC or Langevin dynamics as an auxiliary sequence, the advantage of using Metropolized SGLD as the auxiliary sequence is that it is closer to SGLD  in distribution as its transition distribution also covers the randomness of stochastic gradients, thus can better characterize the convergence behavior of SGLD and lead to sharper convergence guarantees. 
To sum up, we highlight our main contributions as follows:

\begin{itemize}[leftmargin=*]
  \item We provide a new convergence analysis of SGLD for sampling a large class of distributions that can be non-log-concave. In contrast to \cite{raginsky2017non,xu2018global} that require a very large mini-batch size, our convergence guarantee holds for an arbitrary choice of mini-batch size.
  
    %
\item  We prove that SGLD can achieve $\epsilon$-sampling error in total variation distance within $\tilde O(d^4\beta^2\rho^{-4}\epsilon^{-2})$ stochastic gradient evaluations, where $d$ is the problem dimension, $\beta$ is the inverse temperature parameter, and $\rho$ is the Cheeger constant (See Definition~\ref{def:cheeger}) of a  truncated version of the target distribution.
We also prove the convergence of SGLD under the measure of polynomial growth functions, which suggests that the number of required stochastic gradient evaluations is $\tilde O(\epsilon^{-2})$. This improves the state-of-the-art result proved in \cite{xu2018global} by a factor of $\tilde O(\epsilon^{-3})$.
    
    \item We further establish sharper convergence guarantees for SGLD under an additional Hessian Lipschitz condition on the negative log density function $f(\xb)$. We show that $\tilde O(d^{15/4}\beta^{7/4}\rho^{-7/2}\epsilon^{-3/2})$ stochastic gradient evaluations suffice to achieve $\epsilon$-sampling error in total variation distance. Our proof technique is much simpler and more intuitive than existing analysis for proving the convergence of Langevin algorithms under the Hessian Lipschitz condition \citep{dalalyan2017user,mou2019improved,vempala2019rapid}, which can be of independent interest.  
    \end{itemize}




\noindent\textbf{Notation.} 
We use the notation $x\wedge y$ and $x\vee y$ to denote $\min\{x,y\}$ and $\max\{x, y\}$ respectively.
We denote by $\cB(\ub,r)$ the Euclidean of radius $r>0$ centered at $\ub\in\RR^d$. For any distribution $\mu$ and set $\cA$, we use $\mu(\cA)$ to denote the probability measure of $\cA$ under the distribution $\mu$. For any two distributions $\mu$ and $\nu$, we use $\|\mu-\nu\|_{TV}$ and $D_{KL}(\mu,\nu)$ to denote the total variation distance and Kullback–Leibler divergence between $\mu$ and $\nu$ respectively. For $\ub,\vb\in\RR^d$, we use $\transitSGLD_{\ub}(\vb)$ to denote the probability of transiting to $\vb$ after one step SGLD update from $\ub$. Similarly, $\transitSGLD_{\ub}(\cA)$ and $\transitSGLD_{\cA'}(\cA)$ are the probabilities of transiting to a set $\cA\subseteq\RR^d$ after one step SGLD update starting from $\ub$ and the set $\cA'$ respectively. For any two sequences $\{a_n\}$ and $\{b_n\}$, we denote $a_n = O(b_n)$ and $a_n = \Omega(b_n)$ if $a_n\le C_1b_n$ or $a_n\ge C_2 b_n$ for some absolute constants $C_1$ and $C_2$. We use notations $\tilde O(\cdot)$ and $\tilde \Omega(\cdot)$ to hide polylogarithmic factors in $O(\cdot)$ and $\Omega(\cdot)$ respectively.

\section{Related Work}
Markov Chain Monte Carlo (MCMC) methods, such as random walk Metropolis \citep{mengersen1996rates}, ball walk \citep{lovasz1990mixing}, hit-and-run \citep{smith1984efficient} and Langevin algorithms \citep{parisi1981correlation}, have been extensively studied for sampling from a target distribution, and widely used in many machine learning applications. There are a large number of works focusing on developing fast MCMC algorithms and establishing sharp theoretical guarantees. We will review the most related works among them due to the space limit.

Langevin dynamics \eqref{eq:langevin_dynamics} based algorithms have recently aroused as a promising method for accurate and efficient Bayesian sampling in both theory and practice \citep{welling2011bayesian,dalalyan2017theoretical}. 
The non-asymptotic convergence rate of LMC has been extensively investigated in the literature when the target distribution is strongly log-concave \citep{durmus2016sampling,dalalyan2017theoretical,durmus2017nonasymptotic}, weakly log-concave
\citep{dalalyan2017further,mangoubi2019nonconvex}, and non-log-concave but admits certain good isoperimetric properties \citep{raginsky2017non,ma2018sampling,lee2018beyond,xu2018global,vempala2019rapid}, to mention a few. 
The stochastic variant of LMC, i.e., SGLD, is often studied together in the above literature and the convex/nonconvex optimization field \citep{raginsky2017non,zhang2017hitting,xu2018global,gao2018global,chen2018accelerating,deng2020non,chen2020stationary}.
Another important Langevin-based algorithm is the Metropolis Adjusted Langevin Algorithms (MALA)  \citep{roberts1996exponential}, which is developed by introducing a Metropolis-Hasting step into LMC. Theoretically, it has been proved that MALA converges to the target distribution at a linear rate for sampling from both strongly log-concave \citep{dwivedi2018log} and non-log-concave \citep{bou2013nonasymptotic} distributions. 

Beyond first-order MCMC methods, there has also emerged extensive work on high-order MCMC methods. One popular algorithm among them is Hamiltonian Monte Carlo (HMC) \citep{neal2011mcmc}, which introduces a Hamiltonian momentum and leapfrog integrator to accelerate the mixing rate. 
From the theoretical perspective,
\citet{durmus2017convergence} established general conditions under which HMC can be guaranteed to be geometrically ergodic.
\citet{mangoubi2018dimensionally,mangoubi2019nonconvex} proved the convergence rate of HMC for sampling both log-concave and non-log-concave distributions. \citet{bou2018coupling,chen2019fast} studied the convergence of Metropolized HMC (MHMC) for sampling strongly log-concave distributions. Another important high-order MCMC method is built upon the underdamped Langevin dynamics, which incorporates the velocity into the Langevin dynamics \eqref{eq:langevin_dynamics}. 
For continuous-time underdamped Langevin dynamics, its mixing rate has been studied in \cite{eberle2016reflection,eberle2017couplings}.
The convergence of its discrete version has also been widely studied for sampling from both log-concave \citep{chen2017convergence}   
and non-log-concave distributions \citep{chen2015convergence,cheng2018sharp,gao2018global,gao2018breaking}.

\section{Review of the SGLD Algorithm}

For the completeness, we  present the SGLD algorithm \citep{welling2011bayesian} in Algorithm \ref{alg:sgld}, which is built upon the Euler-Maruyama discretization of the continuous-time Langevin dynamics \eqref{eq:langevin_dynamics} while using mini-batch stochastic gradient in each iteration.

In the $k$-th iteration, SGLD samples a mini-batch of data points without replacement, denoted by $\cI$, and  computes the stochastic gradient at the current iterate $\xb_k$, i.e., $\gb(\xb_k,\cI) = 1/B\sum_{i\in\cI}\nabla f_i(\xb_k)$, where $B = |\cI|$ is the mini-batch size. Based on the stochastic gradient, the model parameter is updated using the following rule,
\begin{align*}
\xb_{k+1} = \xb_k - \eta\gb(\xb_k,\cI) + \sqrt{2\eta/\beta}\cdot\bepsilon_k,
\end{align*}
where $\bepsilon_k$ is randomly drawn from a standard normal distribution $N(\zero, \Ib)$ and $\eta>0$ is the step size.

\begin{algorithm}[!t]
	\caption{Stochastic Gradient Langevin Dynamics (SGLD)}
	\label{alg:sgld}
	\begin{algorithmic}
		\STATE \textbf{input:} step size $\eta$; mini-batch size $B$; inverse temperature parameter $\beta$;
	    \STATE Randomly draw $\xb_0$ from initial distribution $\mu_0$.
 		\FOR {$k = 0,1,\ldots, K$}
		\STATE Randomly pick a subset $\cI$ from $\{1,\ldots,n\}$ of size $|\cI|=B$; randomly draw $\bepsilon_k\sim N(\zero,\Ib)$
		\STATE Compute the stochastic gradient $\gb(\xb_k,\cI) = 1/B\sum_{i\in\cI}\nabla f_i(\xb_k)$
		\STATE Update: $\xb_{k+1}=\xb_k-\eta \gb(\xb_k,\cI)+\sqrt{2\eta/\beta}\bepsilon_k$
		\ENDFOR 
		\STATE \textbf{output: $\xb_K$} 
	\end{algorithmic}

\end{algorithm}

\section{Main Results}
In this section, we present our main theoretical results.  We start with the following two definitions. The first one quantifies the goodness of the initial distribution compared with the target distribution, and the second one characterizes the isoperimetric profile of a given distribution. Both definitions are widely used in the convergence analysis of MCMC methods \citep{lovasz1993random,vempala2007geometric,dwivedi2018log,mangoubi2019nonconvex}.

\begin{definition}[$\lambda$-warm start]\label{eq:def_warm_start}
Let $\nu$ be a distribution on $\Omega$. We say the initial distribution $\mu_0$ is a $\lambda$-warm start with respect to $\nu$ if
\begin{align*}
\sup_{\cA:\cA\subseteq\Omega} \frac{\mu_0(\cA)}{\nu(\cA)}\le \lambda. 
\end{align*}
\end{definition}



\begin{definition}[Cheeger constant]\label{def:cheeger} Let $\mu$ be a probability measure on $\Omega$. We say $\mu$ satisfies the isoperimetric inequality with Cheeger constant $\rho$ if for any $\cA\in\Omega$, it holds that
\begin{align*}
\liminf_{h\rightarrow 0^+} \frac{\mu(\cA_h)-\mu(\cA)}{h}\ge \rho\min\big\{\mu(\cA), 1-\mu(\cA)\big\},
\end{align*}
where $\cA_h = \{\xb\in\Omega: \exists \yb\in\cA, \|\xb-\yb\|_2\le h\}$.
\end{definition}


Next, we introduce some common
assumptions on the negative log density function $f(\xb)$ and the stochastic gradient $\gb(\xb,\cI)$. 
\begin{assumption}[Dissipativeness]\label{assump:diss}
There are absolute constants $m>0$ and $b\ge 0$  such that 
\begin{align*}
\la\nabla f(\xb), \xb\ra\ge m\|\xb\|_2^2 - b, \quad\text{for all }\xb\in\RR^d. 
\end{align*}
\end{assumption}
This assumption has been conventionally made in the convergence analysis for sampling from non-log-concave distributions \citep{raginsky2017non,xu2018global,zou2019sampling}. Basically, this assumption implies that the log density function $f(\xb)$ grows like a quadratic function when $\xb$ is outside a ball centered at the origin. Note that a strongly convex function $f(\xb)$  satisfies Assumption~\ref{assump:diss}, but not vice versa.

\begin{assumption}[Smoothness]\label{assump:smooth}
There exists a positive constant $L$ such that for any $\xb,\yb\in\RR^d$ and all functions $f_i(\xb)$, $i=1,\ldots,n$, it holds that
\begin{align*}
\|\nabla f_i(\xb) - \nabla f_i(\yb)\|_2\le L\|\xb - \yb\|_2.
\end{align*}
\end{assumption}
This assumption has also been made in many prior works \citep{raginsky2017non,zhang2017hitting,xu2018global}.

We now define the following function that will be repeatedly used in the subsequent theoretical results:
{
\begin{align}\label{eq:def_barR}
\bar R(z) &= \bigg[\max\bigg\{\frac{625d\log(4/z)}{m\beta},\frac{4d\log(4L/m)+4\beta b}{m\beta}, \frac{4d + 8 \sqrt{d\log(1/z)}+8\log(1/z)}{m\beta}\bigg\}\bigg]^{1/2}.
\end{align}}%
Based on all aforementioned assumptions, we present the convergence result of SGLD in  the following theorem.
\begin{theorem}\label{thm:main_thm}
 For any $\epsilon\in(0,1)$, let $\pi^*\propto e^{-\beta f(\xb)}\ind\big(\xb\in\cB(0,R)\big)$ be the truncated target distribution in $\Omega = \cB(0,R)$ with $R = \bar R(\epsilon K^{-1}/12)$, and $\rho$ be the Cheeger constant of $\pi^*$. Under Assumptions \ref{assump:diss} and \ref{assump:smooth}, we suppose $\PP(\|\xb_0\|_2\le R/2)\le \epsilon/16$, and set the step size as $\eta=\tilde O( \rho^2d^{-2}\beta^{-1}\wedge B^2\rho^2d^{-4}\beta^{-1})$. Then for any $\lambda$-warm start with respect to $\pi$, the output of Algorithm \ref{alg:sgld} satisfies
\begin{align}\label{eq:main_thm_bound}
\|\mu_K^{\text{SGLD}}- \pi\|_{TV}\le   \lambda(1-C_0\eta)^{K}  +\frac{C_1\eta^{1/2}}{B}+C_2\eta^{1/2}+\frac{\epsilon}{2},
\end{align}
where $C_0 = \tilde O\big(\rho^2\beta^{-1}\big)$, 
$C_1 = \tilde O\big(Rd\rho^{-1}\beta^{3/2}\big)$ and $C_2 = \tilde O\big(d\rho^{-1}\beta^{1/2}\big)$ are problem-dependent constants.
\end{theorem}

To prove Theorem \ref{thm:main_thm}, we constructed an auxiliary sequence $\{\xb_k^{\text{MH}}\}_{k=0,1,\ldots}$ by adding a Metropolis-Hasting accept/reject step in each iteration of SGLD. We call this auxiliary sequence  Metropolized SGLD and it is only used in the analysis (please refer to Section \ref{sec:metropolized_SGLD} for the rigorous definition). 
Therefore, the terms in \eqref{eq:main_thm_bound} can be categorized into two types: (1) the approximation error between SGLD iterates $\{\xb_k\}_{k=0,1,\ldots}$ and the auxiliary sequence $\{\xb_k^{\text{MH}}\}_{k=0,1,\ldots}$; and (2) the convergence of Metropolized SGLD to the target distribution $\pi$.

In particular, the four terms on the right-hand side of \eqref{eq:main_thm_bound} are interpreted as follows: the first term corresponds to the sampling error of the auxiliary sequence generated by Metropolized SGLD, which converges to zero at a linear rate. The second to the last terms in \eqref{eq:main_thm_bound} together
reflect the approximation error between SGLD and Metropolized SGLD, which is in the order of $O(\eta^{1/2}+\epsilon)$. More specifically, the second  and third terms correspond to the reject probability of Metropolized SGLD (since SGLD does not have this accept/reject step), which are contributed by the randomness of stochastic gradients and Brownian motion respectively. The last term is related to our choice of $R$ since the constructed Metropolized  restricts  all iterates inside the region $\cB(0,R)$ while SGLD has no constraint on its iterates (see Section \ref{sec:metropolized_SGLD} for more details).


\begin{remark}\label{rmk:comment_cheeger}
For a general non-log-concave distribution, it is difficult to prove a tight bound on the Cheeger constant $\rho$. One possible lower bound of $\rho$ can be obtained via  Buser's inequality \citep{buser1982note,ledoux1994simple}, which shows that the Cheeger constant $\rho$ can be lower bounded by $\Omega(d^{-1/2}c_p)$ under Assumption \ref{assump:smooth}, where $c_p$ is the Poincar\'e constant of the distribution $\stationalRef$. Moreover, \citet{bakry2008simple} gave a simple lower bound of $c_p$, showing that $c_p \ge e^{-\beta\text{Osc}_{R}f}/(2R^2)$, where $\text{Osc}_{R}f = \sup_{\xb\in\cB(\zero,R)}f(\xb) -\inf_{\xb\in\cB(\zero,R)}f(\xb)\le LR^2/2$. Assuming $R = \tilde O(d^{1/2})$, this further implies that $\rho= \Omega(d^{-1})\cdot e^{-O(R^2)} = e^{-\tilde O(d)}$. In addition, better lower bounds of $\rho$ can be proved when the target distribution enjoys better properties. When the target distribution is a mixture of strongly log-concave distributions, the lower bound of $\rho$ can be improved to  $1/\text{poly}(d)$ \citep{lee2018beyond}. Strengthening Assumption \ref{assump:diss} to a local nonconvexity condition yields $\rho=e^{-O(L)}$ \citep{ma2018sampling}.  For log-concave distributions, \citet{lee2017eldan} proved that the Cheeger constant $\rho$ can be lower bounded by $\rho = \Omega\big(1/(\text{Tr}(\bSigma^2))^{1/4}\big)$, where $\bSigma$ is the covariance matrix of the distribution $\stationalRef$. 
When the target distribution is $m$-strongly log-concave, based on \cite{cousins2014cubic,dwivedi2018log}, it can be shown that $\rho = \Omega(\sqrt{m})$.
\end{remark}

Note that the upper bound of the sampling error proved in Theorem \ref{thm:main_thm} relies on the step size, mini-batch size, and the goodness of the initialization (i.e., $\lambda$). In order to guarantee $\epsilon$-sampling error of SGLD, we need to specify the choices of these hyper-parameters. In particular, we present the iteration complexity of SGLD in the following corollary.

\begin{corollary}\label{coro:SGLD1}
Under the same assumptions made in Theorem \ref{thm:main_thm}, we use Gaussian initialization $\mu_0 = N\big(\zero, \Ib/(2\beta L)\big)$. For any mini-batch size $B\le n$ and $\epsilon\in(0,1)$, if we set the step size and the maximum iteration number as
\begin{align*}
\eta &= \tilde O\bigg(\frac{\rho^2\epsilon^2}{d^2\beta} \wedge \frac{B^2\rho^2\epsilon^2}{d^4\beta}\bigg), \\
K&= \tilde O\bigg(\frac{d^3\beta^2}{\rho^4\epsilon^2}\vee \frac{d^5\beta^2}{B^2\rho^4\epsilon^2}\bigg),
\end{align*}
then SGLD can achieve an $\epsilon$ sampling error in total variation distance.      
\end{corollary}
It is worth noting that the iteration complexity in Corollary \ref{coro:SGLD1} holds for any mini-batch size $1\leq B\le n$, as opposed to  \cite{raginsky2017non,xu2018global} that require the mini-batch size to be  $\text{poly}(\epsilon^{-1})$ in order to guarantee vanishing sampling error. Moreover, if we set the mini-batch size to be $B = O(d)$, the number of stochastic gradient evaluations needed to achieve $\epsilon$-sampling error is $K\cdot B=\tilde O(d^4\beta^2\rho^{-4}\epsilon^{-2})$.

Based on Corollary \ref{coro:SGLD1}, we further  prove the convergence of SGLD under the measure of any polynomial growth function. 
\begin{corollary}\label{coro:weak_converge}
Under the same assumptions and hyper-parameter configurations as in Corollary~\ref{coro:SGLD1}, 
let $h(\xb)$ be a polynomial growth function with degree $D$, i.e., $h(\xb)\le C(1+\|\xb\|_2^D)$ for some constant $C$, and $K$ be defined in Corollary \ref{coro:SGLD1}, then the output of SGLD satisfies
\begin{align*}
\EE[h(\xb_K)] - \EE[h(\xb^\pi)]\le C' \epsilon,
\end{align*}
where $\xb^\pi\sim\pi$ denotes the random vector sampled from $\pi$ and $C' = \tilde O\big(d^{D/2}\big)$ is a problem-dependent constant.
\end{corollary}
\begin{remark}
Similar results have been presented in \cite{sato2014approximation,chen2015convergence,vollmer2016exploration,erdogdu2018global}. However, \citet{sato2014approximation} only analyzed the finite-time approximation error between SGLD and the SDE \eqref{eq:langevin_dynamics} rather than the convergence to the target distribution. The convergence results in \cite{chen2015convergence,vollmer2016exploration,erdogdu2018global} also differ from ours as their guarantees are made on the sample path average rather than the last iterate. In addition, these works assume that the Poisson equation solution of the SDE \eqref{eq:langevin_dynamics} has polynomially bounded $i$-th order derivative ($i\in\{2,3,4\}$), which is not required in our result.
\end{remark}

Let us consider a special case that $h(\cdot) = f(\cdot)$, which was studied in \cite{raginsky2017non,xu2018global}. Assumption \ref{assump:smooth} implies that $h(\xb)$ is a quadratic growth function. Then Corollary~\ref{coro:weak_converge} shows that in order to guarantee $\EE[f(\xb_k)]  - \EE[f(\xb^\pi)]\le \epsilon$, SGLD requires $\tilde O(\epsilon^{-2})$ stochastic gradient evaluations. In contrast, in order to achieve the same error,  \citet{raginsky2017non,xu2018global} require $\tilde O(\epsilon^{-8})$ and $\tilde O(\epsilon^{-5})$ stochastic gradient evaluations respectively, both of which are worse than ours.

\section{Improved Convergence Rates under Hessian Lipschitz Condition}

In this section, we will show that the convergence rate of SGLD can be improved if the log density function additionally satisfies the Hessian Lipschitz condition, which is defined as follows.
\begin{assumption}[Hessian Lipschitz]\label{assump:hessian_lip}
There exists a positive constant $H$ such that for any $\xb,\yb\in\RR^d$, it holds that
\begin{align*}
\big\|\nabla^2 f(\xb) - \nabla^2 f(\yb)\big\|_{\textrm{op}}\le H\|\xb - \yb\|_2.
\end{align*}

\end{assumption}
This assumption appears in many recent papers that aim to prove faster convergence rates of LMC \citep{dalalyan2017user,vempala2019rapid,mou2019improved} for sampling from both log-concave and non-log-concave distributions.  

With this additional assumption, we state the convergence result of SGLD in the following theorem.

\begin{theorem}\label{thm:main_thm_hessian}
For any $\epsilon\in(0,1)$, let $\pi^*\propto e^{-\beta f(\xb)}\ind\big(\xb\in\cB(0,R)\big)$ be the truncated target distribution in $\Omega = \cB(0,R)$ with $R = \bar R(\epsilon K^{-1}/12)$, and let $\rho$ be the Cheeger constant of $\pi^*$.
Under Assumptions~\ref{assump:diss}, \ref{assump:smooth}, and \ref{assump:hessian_lip}, we suppose $\PP(\|\xb_0\|_2\le R/2)\le \epsilon/16$. If we set the step size $\eta=\tilde O\big( \rho^2d^{-2}\beta^{-1}B^2\wedge \rho/(d^{3/2}+d\beta^{1/2})\big)$, then for any  $\lambda$-warm start with respect to $\pi$, the output of Algorithm \ref{alg:sgld} satisfies
\begin{align*}
\|\mu_K^{\text{SGLD}}- \pi\|_{TV}\le  \lambda(1-C_0\eta)^{K}  +\frac{C_1\eta^{1/2}}{B}+C_2\eta+ \frac{\epsilon}{2},
\end{align*}
where $C_0 =  O(\beta^{-1}\rho^2)$, $C_1 =  \tilde O(R^2d\rho^{-1}\beta^{3/2})$ and $C_2 =  \tilde O(d^{3/2}\rho^{-1}+Rd^{1/2}\beta \rho^{-1})$ are problem-dependent constants.
\end{theorem}
The four terms in the upper bound in Theorem \ref{thm:main_thm_hessian} have the same interpretation as those in Theorem \ref{thm:main_thm}.
Compared with the convergence result in Theorem \ref{thm:main_thm}, the improvement brought by the Hessian Lipschitz condition lies in the approximation error between the transition distributions of SGLD and Metropolized SGLD, which is improved from $O(\eta^{1/2})$ to $O\big(B^{-1}\eta^{1/2}+\eta\big)$.

Under the same Hessian Lipschitz condition, \citet{dalalyan2017user,mou2019improved,vempala2019rapid} improved the convergence rate of LMC \eqref{eq:def_lmc}. However, \citet{dalalyan2017user} only focused on strongly log-concave distributions and the theoretical results in \citet{mou2019improved,vempala2019rapid} cannot be easily extended to SGLD. 

\begin{corollary}\label{coro:sgld2}
Under the same assumptions made in Theorem \ref{thm:main_thm_hessian}, we use Gaussian initialization $\mu_0 = N\big(\zero, \Ib/(2\beta L)\big)$.  For any mini-batch size $B\le n$,  if 
set the step size and maximum iteration number as
\begin{align*}
\eta &= \tilde O\bigg(\frac{\rho^2B^2\epsilon^2}{d^2\beta}\wedge \frac{\rho\epsilon}{d^{3/2}+d\beta^{1/2}}\bigg),\\
K&= \tilde O\bigg(\frac{d^5\beta^2}{\rho^4B^2\epsilon^2}+\frac{d^{5/2}\beta + d^2\beta^{3/2}}{\rho^3\epsilon}\bigg),
\end{align*}
then SGLD can achieve an $\epsilon$ sampling error in terms of total variation distance.      
\end{corollary}

Note that the required number of stochastic gradient evaluations is $K\cdot B = \tilde O\big(d^5\beta^{2}/(B\rho^4\epsilon^2)+Bd^{5/2}\beta^{3/2}/(\rho^3\epsilon)\big)$. Therefore, if setting the mini-batch size as $B=\tilde O\big([d^{5/2}\beta^{1/2}\rho\epsilon]^{1/2}\big)$, it can be derived that the gradient complexity of SGLD is $\tilde O(d^{15/4}\beta^{7/4}\rho^{-7/2}\epsilon^{-3/2})$. This strictly improves the stochastic gradient complexity (i.e., number of stochastic gradient evaluations to achieve $\epsilon$-sampling error) of SGLD without Assumption \ref{assump:hessian_lip} by a factor of $\tilde O(d^{1/4}\beta^{1/4}\rho^{-1/2}\epsilon^{-1/2})$.

\section{Proof Outline}\label{sec:proof_main}
In this section, we will sketch the proof of the main results (Theorem \ref{thm:main_thm}). The missing proofs for the other theorems, corollaries and lemmas are deferred to the appendix. We first highlight the key proof technique and its novelty and difference compared with prior works. Then we will go over each of the key steps in detail. 



\subsection{Proof Technique and Novelty}

\paragraph{Proof Technique.} 
Our proof relies on two sequences (green arrows in Figure \ref{fig:decomposition}): \textbf{Projected SGLD} ($\xb_k^{\text{\tiny Proj-SGLD}}$) and \textbf{Metropolized SGLD} ($\xb_k^{\text{\tiny MH}}$). Projected SGLD is constructed by adding an accept/reject step to the standard SGLD algorithm, which was first studied in \citet{zhang2017hitting}. Metropolized SGLD is a ``virtual'' sequence constructed by further adding a Metropolis Hasting step into Projected SGLD (the Metropolis Hasting step is computationally intractable so that Metropolized SGLD is not a practical algorithm and we only use it for theoretical analysis). Due to such Metropolis Hasting step, Metropolized SGLD is a time-reversible Markov chain and thus enjoys good conductance properties. Based on these two auxiliary sequences, we will prove the convergence of SGLD following three steps: (1) show that the output of Projected SGLD is close to that of SGLD in distribution (see Lemma \ref{lemma:connection_SGLD}); (2) show that the transition distribution of Projected SGLD is close to that of Metropolized SGLD (see Lemma \ref{lemma:approximation}); and (3) prove the convergence of Projected SGLD based on the conductance of Metropolized SGLD (see Lemma \ref{lemma:convergence_approximate}). 

\paragraph{Technical Novelty. }
In order to prove the convergence rate of SGLD, prior works \citep{raginsky2017non,xu2018global} typically make use of the LMC iterates $\xb_k^{\text{\tiny LMC}}$ and decompose the sampling error of SGLD (the error between $\xb_k$ and $\xb^\pi$) into two parts: (1) the error between SGLD iterates and LMC iterates; and (2) the sampling error of LMC (though \citet{raginsky2017non,xu2018global} bound the sampling error of $\xb_k^{\text{\tiny LMC}}$ in different ways). We illustrate the roadmap of different proof techniques in  Figure \ref{fig:decomposition}.
Note that their results on the error between $\xb_k$ and $\xb_k^{\text{\tiny LMC}}$ diverge as $k$ increases, due to the uncertainty of stochastic gradients. This suggests that LMC may not be a good enough auxiliary chain for studying SGLD. In contrast, our constructed auxiliary sequences (i.e., Projected SGLD and Metropolized SGLD) are closer to SGLD since they also cover the randomness of stochastic gradients (this randomness can be included as part of the transition distribution, see Section \ref{sec:metropolized_SGLD} for more details). Therefore, our proof technique can lead to a sharper convergence analysis than those in \citet{raginsky2017non,xu2018global}, which consequently gives a faster convergence rate of SGLD for sampling from non-log-concave distributions.



\begin{figure}[!t]
    \centering
    \includegraphics[scale=0.4]{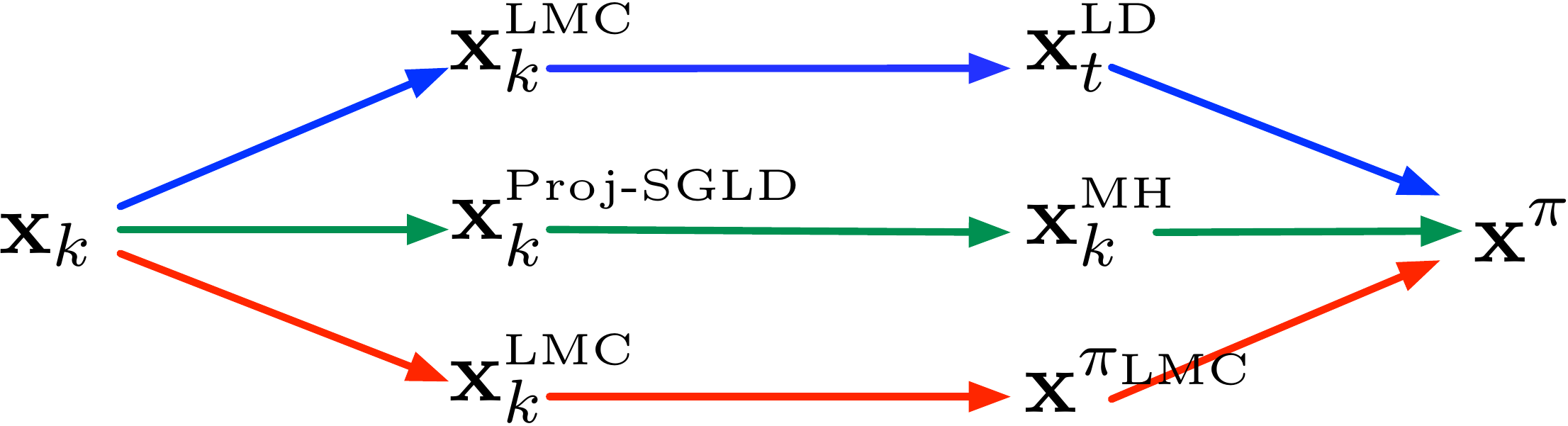}
    \caption{Illustration of the analysis framework of SGLD in different works:  \textcolor{blue}{\textbf{Raginsky et al. (2017)}}, \textcolor{red}{\textbf{Xu et al. (2018)}}, \textcolor{ForestGreen}{\textbf{this work}}. The goal is to prove the convergence of SGLD iterates $\xb_k$ to the point following the target distribution $\xb^\pi$. Note that, $\xb_{k}^{\text{\tiny LMC}}$, $\xb_{k}^{\text{\tiny Proj-SGLD}}$ and $\xb_{k}^{\text{\tiny MH}}$ denote the $k$-th iterates of LMC, Proj-SGLD, and Metropolized SGLD respectively; $\xb_t^{\text{\tiny LD}}$ denotes the solution of  \eqref{eq:langevin_dynamics} at time $t$; $\xb^{\pi_{\text{\tiny LMC}}}$ denotes the point following the stationary distribution of LMC. }
    \label{fig:decomposition}
\end{figure}

We would also like to point out that while the construction of Metropolized SGLD follows the same spirit of \citet{zhang2017hitting}, it has a different goal and thus the corresponding analysis is not the same. Specifically, \citet{zhang2017hitting} only characterizes the hitting time of SGLD to a certain set by lower bounding the restricted conductance of SGLD, but does not prove its convergence to $\pi$. In contrast, we focus on the ability of SGLD for sampling from a certain target distribution. Thus we not only need to analyze the conductance of SGLD, but also need to bound the approximation error between the distribution of $\xb_k$ and the target one (see Lemmas \ref{lemma:convergence_approximate} and \ref{lemma:contraction} and their proofs for more details), which is more challenging.
As a consequence, we prove that the sampling error of SGLD to the target distribution can be upper bounded by $O(\sqrt{\eta})$, while the analysis in \citet{zhang2017hitting} can only give $O(1)$ sampling error.

\subsection{Projected SGLD and Its Equivalence to SGLD}

Projected SGLD is constructed by adding an extra step in Algorithm \ref{alg:projected_sgld} with the following accept/reject rule:
\begin{align}\label{eq:def_alg_accept}
\begin{split}
\xb_{k+1} =
    \begin{cases}
    \xb_{k+1} & \xb_{k+1}\in\cB(\xb_{k},r)\cap\cB(\zero,R); \\
    \xb_k & \mbox{otherwise}.
    \end{cases}
    \end{split}
\end{align}
This step ensures each new iterate $\xb_{k+1}$ does not go too far away from the current iterate and all iterates are restricted in a (relatively) large region $\cB(\zero,R)$. The entire algorithm is summarized in Algorithm \ref{alg:projected_sgld}.
Due to the above accept/reject rule, Projected SGLD is slightly different from the standard SGLD algorithm (see Algorithm \ref{alg:sgld}). However, we can show that Projected SGLD is nearly the same as SGLD given proper choices of $R$ and $r$. In particular, in the following lemma, we will show that the total variation distance between the distributions of the outputs of both algorithms can be arbitrarily small.

\begin{algorithm}[!t]
	\caption{Projected SGLD }
	\label{alg:projected_sgld}
	\begin{algorithmic}
		\STATE \textbf{input:} step size $\eta$; mini-batch size $B$; inverse temperature parameter $\beta$; radius $R$, $r$; 
	    \STATE Randomly draw $\xb_0$ from initial distribution $\mu_0$.
 		\FOR {$k = 0,1,\ldots, K$}
		\STATE Randomly pick a subset $\cI$ from $\{1,\ldots,n\}$ of size $|\cI|=B$; randomly draw $\bepsilon_k\sim N(\zero,\Ib)$
		\STATE Compute the stochastic gradient $\gb(\xb_k,\cI) = 1/B\sum_{i\in\cI}\nabla f_i(\xb_k)$
		\STATE Update: $\xb_{k+1}=\xb_k-\eta \gb(\xb_k,\cI)+\sqrt{2\eta/\beta}\bepsilon_k$
		\IF{$\xb_{k+1}\not\in \cB(\xb_k,r)\cap\cB(\zero,R)$} 
		\STATE $\xb_{k+1}= \xb_k$
		\ENDIF
		\ENDFOR 
		\STATE \textbf{output: $\xb_K$} 
	\end{algorithmic}

\end{algorithm}

\begin{lemma}\label{lemma:connection_SGLD}
Let $\mu_K^{\text{SGLD}}$ and $\mu_K^{\text{Proj-SGLD}}$ be  distributions of the outputs of  standard SGLD (Algorithm \ref{alg:sgld}) and projected SGLD (Algorithm \ref{alg:projected_sgld}). For any $\epsilon\in(0,1)$, we set
\begin{align*}
R = \bar R(\epsilon K^{-1}/4),\quad
r = \sqrt{2\eta d /\beta}\big(2+\sqrt{2\log(8K/\epsilon)/d}\big).
\end{align*}
Suppose $\PP(\|\xb_0\|_2\le R/2)\le \epsilon/16$ and setting $\eta\le (LR+G)^{-2}\beta^{-1}d$, then we have
\begin{align*}
\big\|\mu_K^{\text{SGLD}}-\mu_K^{\text{Proj-SGLD}}\big\|_{TV}\le \frac{\epsilon}{4}.
\end{align*}
\end{lemma}


\subsection{Construction of Metropolized SGLD}\label{sec:metropolized_SGLD}
Projected SGLD will approximately generate samples from the following truncated target distribution since it restricts all iterates to the region $\Omega:=\cB(0,R)$,
\begin{align}\label{eq:def_restrict_target}
\begin{split}
\stationalRef(\dd \xb) =   \begin{cases}
    \frac{e^{-\beta f(\xb)}}{\int_{\Omega}e^{-\beta f(\yb)}\dd\yb}\dd\xb & \xb\in\Omega; \\
    0 & \mbox{otherwise}.
    \end{cases}
\end{split}
\end{align}
Then we will  characterize the convergence of Projected SGLD to $\pi^*$. In particular, we will introduce a useful auxiliary Markov chain called Metropolized SGLD, i.e., SGLD with a Metropolis-Hasting step. We will first give the transition distribution of the Markov chain corresponding to Projected SGLD.





\noindent\textbf{Transition distribution of Projected SGLD.} Let $\gb\big(\xb,\cI\big)$ be the stochastic gradient computed at the point $\xb$, where $\cI$ denotes the  mini-batch of data points queried in the stochastic gradient computation. Then it is clear that Algorithm \ref{alg:projected_sgld} can be described as a Markov process. More specifically, let $\ub$ and $\wb$ be the starting point and the point obtained after one-step iteration of Algorithm \ref{alg:projected_sgld}, the Markov chain in this iteration can be formed as $\ub\rightarrow \vb\rightarrow \wb$, where $\vb$ is generated based on the following conditional probability density function,
\begin{align}\label{eq:trans_SGLD}
P(\vb|\ub)= \EE_{\cI}[P(\vb|\ub,\cI)] = \EE_{\cI} \bigg[\frac{1}{(4\pi\eta/\beta)^{d/2}}\exp\bigg(-\frac{\|\vb - \ub + \eta\gb(\ub,\cI)\|_2^2}{4\eta/\beta}\bigg)\bigg|\ub\bigg],
\end{align}%
which is exactly the transition probability of standard SGLD (i.e., without any accept/reject step). Let $R>0$ be a tunable radius and recall that $\Omega = \cB(\zero,R)$. The process $\vb\rightarrow \wb$ can be formulated as
\begin{align}\label{eq:markov_chain_v2w}
\wb =
    \begin{cases}
    \vb & \vb\in\cB(\ub,r)\cap\Omega; \\
    \ub & \mbox{otherwise}.
    \end{cases}
\end{align}
Let $p(\ub) = \PP_{\vb\sim P(\cdot|\ub)}[\vb\in\cB(\ub,r)\cap\Omega]$ be the acceptance probability in \eqref{eq:markov_chain_v2w}, and $Q(\wb|\ub)$ be the conditional PDF that describes $\ub\rightarrow \wb$. 
Then we have
\begin{align*}
Q(\wb|\ub) &= (1-p(\ub))\delta_{\ub}(\wb)+ P(\wb|\ub)\cdot\ind\big[\wb\in\cB(\ub,r)\cap\Omega\big],
\end{align*}
where $P(\wb|\ub)$ is computed by replacing $\vb$ with $\wb$ in \eqref{eq:trans_SGLD}.
Similar to \cite{zhang2017hitting,dwivedi2018log}, we consider the $1/2$-lazy version of the above Markov process, i.e., a Markov process with the following  transition distribution
\begin{align}\label{eq:def_trans_lz_sgld}
\transitSGLD_{\ub}(\wb) = \frac{1}{2}\delta_{\ub}(\wb) + \frac{1}{2}Q(\wb|\ub),
\end{align}
where $\delta_\ub(\cdot)$ is the Dirac-delta distribution at $\ub$.
However, it is difficult to directly prove the ergodicity of the Markov process with transition distribution $\transitSGLD_\ub(\wb)$, and it is also hard to tell whether its stationary distribution exists or not. Besides, SGLD is known to be asymptotically biased \citep{teh2016consistency,vollmer2016exploration}, which does not converge to the target distribution $\pi$ even when it runs for infinite steps. It remains unclear whether Projected SGLD can converge to the target distribution given the formula of its transition distribution.

\noindent\textbf{Metropolized SGLD.}
In order to quantify the sampling error for the output of Projected SGLD in Algorithm \ref{alg:projected_sgld} and prove its convergence, we follow the  idea of \cite{zhang2017hitting}, which constructs an auxiliary Markov process by adding an extra Metropolis-Hasting correction step into Algorithm \ref{alg:projected_sgld}. We call it Metropolized SGLD.
Given the starting point $\ub$, let $\wb$ be the candidate state generated from the distribution $\transitSGLD_\ub(\cdot)$. Metropolized SGLD will accept the candidate $\wb$ with the following probability, 
\begin{align*}
\alpha_{\ub}(\wb) = \min\bigg\{1, \frac{\transitSGLD_{\wb}(\ub)}{\transitSGLD_{\ub}(\wb)}\cdot\exp\big[-\beta\big(f(\wb) - f(\ub)\big)\big]\bigg\}.
\end{align*}
Let $\transitRef_\ub(\cdot)$ denote the transition distribution of such auxiliary Markov process, i.e.,
\begin{align*}
\transitRef_\ub(\wb) = (1-\alpha_\ub(\wb))\delta(\ub) + \alpha_\ub(\wb)\transitSGLD_\ub(\wb).
\end{align*}
It is easy to verify that the aforementioned Markov process is time-reversible. Due to this Metropolis-Hastings correction step, the Markov chain can converge to a unique stationary distribution $\stationalRef\propto e^{-\beta f(\xb)}\cdot\ind(\xb\in\Omega)$ \citep{zhang2017hitting}. It is worth pointing out that Metropolized SGLD cannot be implemented in practice since we are only allowed to query a subset of the training data in each iteration of SGLD, thus we are not able to 
precisely calculate the accept probability $\alpha_\ub(\wb)$, which involves the expectation computation over the stochastic mini-batch of data points.
Nevertheless, we will only use this auxiliary Markov chain in our theoretical analysis to show the convergence of Algorithm \ref{alg:projected_sgld}.

We will further show that the transition distribution of Projected SGLD ($\cT_\ub(\cdot)$) can be $\delta$-close to that of Metropolized SGLD ($\cT_\ub^*(\cdot)$) for some small quantity $\delta$ governed by $\eta$, which is provided in the following lemma. 
\begin{lemma}\label{lemma:approximation}
Under Assumption \ref{assump:smooth}, let $G = \max_{i\in[n]} \|\nabla f_i(\zero)\|_2$, and set $r = \sqrt{10\eta d/\beta}\big(1+\sqrt{\log(8K/\epsilon)/d}\big)$, where $K$ is the total number of iterations of Projected SGLD. Then 
there exists a constant \begin{align*}
\delta &= \Big[10Ld\eta  +10L(LR+G)d^{1/2}\beta^{1/2}\eta^{3/2}+ 12\beta(LR+G)^2 d\eta/B+ 2\beta^2(LR+G)^4\eta^2/B\Big]\notag\\
&\qquad \cdot \big(1+\sqrt{\log(8K/\epsilon)/d}\big)^2
\end{align*}
such that for any set $\cA\subseteq \Omega$ and any point $\ub\in \Omega$,
\begin{align}\label{eq:delta_approximation}
(1-\delta)\transitRef_\ub(\cA)\le \transitSGLD_\ub(\cA)\le (1+\delta)\transitRef_\ub(\cA). 
\end{align}
\end{lemma}


\subsection{Convergence of Projected SGLD}
In this part, we will characterize the convergence of Projected SGLD, which consists of two steps:
(1) given the $\delta$-closeness result in Lemma \ref{lemma:approximation}, we prove that Projected SGLD can converge to the truncated target distribution $\stationalRef$ up to some approximation error determined by $\delta$; and (2) we prove that with a proper choice of the truncation radius $R$, the total variation distance between $\stationalRef$ and the target distribution $\pi$ can be sufficiently small.



\noindent\textbf{Convergence of Projected SGLD to $\stationalRef$.}
We first provide the definition of the conductance for a time-reversible Markov chain as follows.
\begin{definition}[Conductance]\label{def:s-conductance}
The conductance of a time-reversible Markov chain with transition distribution $\transitRef_\ub(\cdot)$ and stationary distribution $\stationalRef$ is defined by,
\begin{align*}
\phi: = \inf_{\cA: \cA\subseteq\Omega, \stationalRef(\cA)\in(0,1)}\frac{\int_{\cA}\transitRef_\ub(\Omega\bs\cA)\stationalRef(\dd \ub)}{\min\{\stationalRef(\cA), \stationalRef(\Omega\bs\cA)\}},
\end{align*}
where $\Omega$ is the support of the state of the Markov chain.
\end{definition}
In Lemma \ref{lemma:approximation}, we have already shown that the transition distribution of Algorithm \ref{alg:projected_sgld}, i.e., $\transitSGLD_\ub(\cdot)$ is $\delta$-close to that of Metropolized SGLD, i.e., $\transitRef_\ub(\cdot)$, for some small quantity $\delta$. 
Besides, from \cite{lovasz1993random,vempala2007geometric}, we know that a time-reversible Markov chain can converge to its stationary distribution at a linear rate depending on its conductance.
Therefore, we aim to characterize the convergence rate of $\transitSGLD_\ub(\cdot)$ based on the ergodicity of $\transitRef_\ub(\cdot)$.
We utilize the conductance parameter of $\transitRef_\ub(\cdot)$, denoted by $\phi$, and establish the convergence of $\transitSGLD_\ub(\cdot)$ in total variation distance in the following lemma.
\begin{lemma}\label{lemma:convergence_approximate}
Let $\mu_K^{\text{Proj-SGLD}}$ be the distribution of the output of Algorithm \ref{alg:projected_sgld}. Under Assumption~\ref{assump:smooth}, if $\transitSGLD_\ub(\cdot)$ is $\delta$-close to $\transitRef_\ub(\cdot)$ with $\delta\le\min\{1-\sqrt{2}/2, \phi/16\}$, then  for any  $\lambda$-warm start initial distribution with respect to $\stationalRef$, it holds that 
\begin{align*}
\|\mu_K^{\text{Proj-SGLD}} -\stationalRef\|_{TV}\le \lambda\big(1 - \phi^2/8\big)^K + 16\delta/\phi.    
\end{align*}
\end{lemma}
Lemma \ref{lemma:convergence_approximate} shows that Projected SGLD converges to $\stationalRef$ in total variance distance with approximation error up to $16\delta/\phi$. The next step is to characterize the conductance parameter $\phi$ and reveal its dependency on the problem-dependent parameters, which we state in the following lemma.
\begin{lemma}\label{lemma:lowerbound_phi}
Under Assumptions \ref{assump:diss} and \ref{assump:smooth}, if the step size satisfies $\eta\le \big[35(Ld+(LR+G)^2\beta d/B)\big]^{-1}\wedge [25\beta(LR+G)^2]^{-1}$, there exists an absolute constant $c_0$ such that
\begin{align*}
\phi\ge c_0\rho\sqrt{\eta/\beta},
\end{align*}
where $\rho$ is the Cheeger constant of the distribution $\stationalRef$. 
\end{lemma}
\noindent\textbf{Bounding the difference between $\pi$ and $\stationalRef$.} 
Lemmas \ref{lemma:convergence_approximate} and \ref{lemma:lowerbound_phi} together guarantee that Algorithm \ref{alg:projected_sgld}  converges to the truncated target distribution $\stationalRef$. Thus the last thing remaining to be done is ensuring that $\stationalRef$ is sufficiently close to $\pi$. The following lemma characterizes the total variation distance between the target distribution $\pi$ and its truncated version $\pi^*$ in $\cB(\zero,R)$.

\begin{lemma}\label{lemma:approximate_target}
For any $\epsilon\in(0,1)$, set $R=\bar R(\epsilon/12)$ and let $\Omega = \cB(\zero,R)$ and $\stationalRef$ be the truncated target distribution in $\Omega$. Then the total variation distance between $\stationalRef$ and $\pi$ can be upper bounded by
$\|\stationalRef - \pi\|_{TV}\le \epsilon/4$.
\end{lemma}
\begin{proof}[Proof of Theorem \ref{thm:main_thm}]
The rest proof of Theorem \ref{thm:main_thm} is straightforward by combining Lemmas~\ref{lemma:connection_SGLD},~\ref{lemma:convergence_approximate}, and \ref{lemma:approximate_target} using the triangle inequality. We defer the detailed proof to Appendix \ref{sec:proof_remaining}.
\end{proof}

\section{Conclusion}
In this paper, we proved a faster  convergence rate of SGLD for sampling from a broad class of distributions that can be non-log-concave. 
In particular, we developed a new proof technique for characterizing the convergence of SGLD. Different from the existing works that mainly study the convergence of SGLD using full-gradient-based Markov chains such as LMC or continuous Langevin dynamics,
the key of our proof technique relies on two auxiliary Markov chains: Projected SGLD and Metropolized SGLD, which can better capture the behavior of SGLD since they also cover the randomness of the stochastic gradients. Our proof technique is of independent technical interest and can be potentially adapted to study the convergence of other stochastic gradient-based sampling algorithms.


\appendix



\section{Proofs of the Main Results}\label{sec:proof_remaining}
In this section, we present the detailed proofs of our main theorems and corollaries.


\subsection{Proof of Theorem \ref{thm:main_thm}}
Now we provide the detailed proof of Theorem \ref{thm:main_thm} based on the key lemmas presented in our proof roadmap. 
\begin{proof}[Proof of Theorem \ref{thm:main_thm}]
We first characterize the condition on the step size required in Lemmas~\ref{lemma:approximation},~\ref{lemma:convergence_approximate} and \ref{lemma:lowerbound_phi}. From Lemma \ref{lemma:approximation}, we know that if $\eta\le [25\beta(LR+G)^2]^{-1}$ and $\beta\ge 1$, the transition distribution $\transitSGLD_\ub(\cdot)$ can be $\delta$-close to $\transitRef_\ub(\cdot)$ with
\begin{align}\label{eq:choose_approximation_para}
\delta &= \big[10Ld\eta  +10L(LR+G)d^{1/2}\beta^{1/2}\eta^{3/2} + 12\beta (LR+G)^2 d\eta/B+ 2\beta^2(LR+G)^4\eta^2/B\big]\notag\\
&\qquad\cdot \bigg(1+\sqrt{\frac{\log(8K/\epsilon)}{d}}\bigg)^2
\notag\\
&\le \big[14Ld\eta + 14(LR+G)^2\beta d\eta/B\big]\cdot\bigg(1+\sqrt{\frac{\log(8K/\epsilon)}{d}}\bigg)^2.
\end{align}
Besides, note that Lemma \ref{lemma:convergence_approximate} requires $\delta\le \min\{1-\sqrt{2}/2,\phi/16\}$, which can be satisfied if 
\begin{align*}
\big[14Ld\eta + 14(LR+G)^2\beta d\eta/B\big]\cdot\bigg(1+\sqrt{\frac{\log(8K/\epsilon)}{d}}\bigg)^2\le\min\{1-\sqrt{2}/2,\phi/16\}. 
\end{align*}
Then based on the requirement of $\eta$  and the lower bound of $\phi$ in Lemma \ref{lemma:lowerbound_phi}, it suffices to set the step size to be 
\begin{align*}
\eta &\le \min\bigg\{\frac{1}{25\beta(LR+G)^2},\frac{1}{35(Ld+(LR+G)^2\beta d/B)},\bigg(\frac{c_0\rho}{16\sqrt{\beta}\big(14Ld + 14(LR+G)^2\beta d/B\big)}\bigg)^2\bigg\}\notag\\
&\qquad\cdot \bigg(1+\sqrt{\frac{\log(8K/\epsilon)}{d}}\bigg)^{-4}.
\end{align*}
Now we are able to put the results of these lemmas together to establish the convergence of Algorithm~\ref{alg:projected_sgld}. Note that if $\mu_0$ is a $\lambda$-warm start to $\pi$, it must be a $\lambda$-warm start to $\pi^*$ since $\pi^*(\cA)\ge \pi(\cA)$ for any $\cA\in\Omega$. Then Lemma \ref{lemma:convergence_approximate} applies.
Combining Lemmas \ref{lemma:connection_SGLD},\ref{lemma:convergence_approximate} and \ref{lemma:approximate_target} and setting $R = \bar R(\epsilon/12)$ for arbitrary $\epsilon\in(0,1/2]$, we have
\begin{align*}
\|\mu_k^{\text{SGLD}} - \pi\|_{TV}&\le \|\pi - \stationalRef\|_{TV} + \|\mu_k^{\text{Proj-SGLD}} - \stationalRef\|_{TV} + \|\mu_k^{\text{SGLD}} - \mu_k^{\text{Proj-SGLD}}\|_{TV} \notag\\
&\le \frac{\epsilon}{2} + \lambda\big(1 - \phi^2/8\big)^k + \frac{16\delta}{\phi} \notag\\
&\le \frac{\epsilon}{2} + \lambda(1-C_0\eta)^{k} + (C_1B^{-1}+C_2)\eta^{1/2},
\end{align*}
where $C_0 = c_0^2\rho^2/(8\beta)$, $C_1 = 224\big(1+\sqrt{\log(8K/\epsilon)/d}\big)^4(LR+G)^2\beta^{3/2} d\rho^{-1}/c_0$, $C_2 = 224\big(1+\sqrt{\log(8K/\epsilon)/d}\big)^4Ld\beta^{1/2}\rho^{-1}/c_0$ are problem-dependent constants. This completes the proof.
\end{proof}
\subsection{Proof of Corollary \ref{coro:SGLD1}}
We first present the following technical lemma.
\begin{lemma}\label{lemma:quadratic_lower_bound}
Under Assumption \ref{assump:diss}, the objective function $f(\xb)$ satisfies
\begin{align*}
f(\xb)\ge \frac{m}{4}\|\xb\|_2^2 + f(\xb^*) -\frac{b}{2}.
\end{align*}
\end{lemma}

Now we prove Corollary \ref{coro:SGLD1}.
\begin{proof}[Proof of Corollary \ref{coro:SGLD1}]
The first step is to characterize the quantity of $\lambda$.
Direct calculation gives
\begin{align*}
\frac{\mu_0(\dd\xb)}{\stationalRef(\dd\xb)}\le  \frac{\int_{\RR^d} e^{-\beta f(\yb)}\dd \yb}{\int_{\RR^d} e^{-\beta L\|\yb\|_2^2}\dd \yb} \cdot e^{-\beta L\|\xb\|_2^2+\beta f(\xb)}.
\end{align*}
By Assumption \ref{assump:smooth}, we have
\begin{align*}
f(\xb)\le f(\xb^*) + \frac{L}{2}\|\xb - \xb^*\|_2^2 \le f(\xb^*) + L\|\xb^*\|_2^2 + L\|\xb\|_2^2,
\end{align*}
which implies that
\begin{align*}
e^{-\beta L\|\xb\|_2^2+\beta f(\xb)} \le e^{\beta [f(\xb^*) + L\|\xb^*\|_2^2]}.
\end{align*}
Moreover, by Lemma \ref{lemma:quadratic_lower_bound}, we have
\begin{align*}
 \int_{\RR^d}e^{-\beta f(\yb)}\dd \yb\le e^{-\beta[f(\xb^*)-b/2]}\int_{\RR^d}e^{-m\beta\|\yb\|_2^2/4}\dd\yb=
\bigg(\frac{4\pi}{ m\beta}\bigg)^{d/2}e^{-\beta[f(\xb^*)-b/2]}.
\end{align*}
Besides, we have
\begin{align*}
\int_{\RR^d}e^{-\beta L\|\yb\|_2^2}\dd \yb=\bigg(\frac{\pi}{ L\beta}\bigg)^{d/2}.
\end{align*}
Combining the above results, we can get
\begin{align}\label{eq:upperbound_lambda}
\lambda  \le \max_{\xb}\frac{\mu_0(\dd\xb)}{\stationalRef(\dd\xb)} \le \bigg(\frac{4L}{m}\bigg)^{d/2}e^{\beta [L\|\xb^*\|_2^2+b/2]} = e^{O(d)}.
\end{align}
In order to ensure that the sampling error $\|\mu_k  - \pi\|_{TV}$ is smaller than $\epsilon$, it suffices to choose $\eta$ and $k$ such that
\begin{align*}
\lambda(1-C_0\eta)^k= \frac{\epsilon}{4},\quad C_1B^{-1}\eta^{1/2}+C_2\eta^{1/2}= \frac{\epsilon}{4}.
\end{align*}
Note that we have $R = \bar R(\epsilon/12) = \tilde O(d^{1/2}\beta^{-1/2})$. Then it follows that $C_0 = O(\rho^2\beta^{-1})$, $C_1 = \tilde O(d^2\rho^{-1}\beta^{1/2})$ and $C_2 = \tilde O(d\rho^{-1}\beta^{1/2})$. Plugging these into the above equation immediately implies that
\begin{align*}
\eta = \tilde{O}\bigg(\frac{\rho^2\epsilon^2}{d^2\beta} \wedge \frac{B^2\rho^2\epsilon^2}{d^4\beta}\bigg)\ \text{and} \ K = O\bigg(\frac{\log(\lambda/\epsilon)}{C_0\eta}\bigg) = \tilde O\bigg(\frac{d^3\beta^2}{\rho^4\epsilon^2}\vee \frac{d^5\beta^2}{B^2\rho^4\epsilon^2}\bigg),
\end{align*}
which completes the proof.
\end{proof}

\subsection{Proof of Corollary \ref{coro:weak_converge}}
\begin{proof}[Proof of Corollary \ref{coro:weak_converge}]
Let $\mu_k$ be the distribution of the SGLD iterate $\xb_k$ and denote $\tilde \Omega = \cB(\zero,\tilde R)$ for some $\tilde R$ which we will specify later, then it holds that,
\begin{align*}
\EE[h(\xb_K)] - \EE[h(\xb^\pi)] &= \int_{\tilde \Omega} h(\xb) \mu_K(\dd\xb) - \int_{\RR^d} h(\xb) \pi(\dd\xb)\notag\\
&\le\bigg|\int_{\tilde\Omega} h(\xb) \mu_K(\dd\xb) - \int_{\tilde\Omega} h(\xb) \pi(\dd\xb) \bigg| + \int_{\RR^d\bs\tilde\Omega} h(\xb) \pi(\dd\xb).
\end{align*}
Note that $h(\xb)$ is a polynomial growth function with degree $D$, thus by definition, for all $\xb\in\tilde\Omega$, we have
\begin{align*}
h(\xb)\le C(1+\|\xb\|_2^D),
\end{align*}
for some absolute constant $C$. Then by Corollary \ref{coro:SGLD1}, we know that $\|\mu_K-\pi\|_{TV}\le \epsilon$. Thus it follows that
\begin{align*}
\bigg|\int_{\tilde\Omega} h(\xb) \mu_K(\dd\xb) - \int_{\tilde\Omega} h(\xb) \pi(\dd\xb) \bigg| \le C(1+\tilde R^D)\bigg|\int_{\tilde\Omega} \mu_K(\dd\xb) - \int_{\tilde\Omega}  \pi(\dd\xb) \bigg|\le C(1+\tilde R^D)\epsilon.
\end{align*}
The rest of the proof will be proving the upper bound of $\int_{\RR^d\bs\tilde\Omega}h(\xb)\pi(\dd \xb)$.
We first introduce an auxiliary distribution defined by
\begin{align*}
q(\xb) = \frac{e^{-m\beta\| \xb\|_2^2/8}}{(8\pi /(m\beta))^{d/2}}.
\end{align*}
Note that the stationary distribution $\pi$ takes form
\begin{align*}
\pi(\dd\xb) = \frac{e^{-\beta f(\xb)}}{Z}\dd\xb,
\end{align*}
where $Z= \int_{\RR^d}e^{-\beta f(\xb)}\dd \xb$ is the normalization coefficient. By \citet{raginsky2017non} ((3.21) in Section 3.5), we know  that under Assumption \ref{assump:smooth}, it holds that  $Z\ge \exp(-\beta f(\xb^*))\cdot [2\pi/(\beta L)]^{d/2}$. Then it is clear that if
\begin{align*}
e^{-\beta\big[f(\xb) + m\|\xb \|_2^2/8\big]}\le \exp\big(-\beta f(\xb^*)\big)\cdot \bigg(\frac{m}{4L}\bigg)^{d/2},
\end{align*}
we have $\stationalRef(\xb)\le q(\xb)$. By Lemma \ref{lemma:quadratic_lower_bound}, we know that 
\begin{align}\label{eq:lowerbound_function_quadratic}
-f(\xb) + \frac{m}{8}\|\xb\|_2^2\le \frac{b}{2} - f(\xb^*) - \frac{m}{8}\|\xb\|_2^2.
\end{align}
Therefore, it can be guaranteed that $\pi(\xb)\le q(\xb)$ if $\|\xb\|_2^2\ge 4m^{-1}(\beta^{-1}d\log(4L/m)+b)$. Therefore, for any $\tilde R^2\ge 4m^{-1}\big(\beta^{-1}d\log(4L/m)+b\big)$ it holds that,
\begin{align*}
\int_{\RR^d\bs\tilde \Omega} h(\xb) \pi(\dd\xb)&\le \frac{1}{[8\pi/(m\beta)]^{d/2}}\int_{\|\xb\|_2^2\ge \tilde R^2}C(1+\|\xb\|_2^D)\cdot \exp\big(-m\beta\|\xb\|_2^2/8\big)\dd \xb \notag\\
&\le 2C\int_{x\ge m\beta \tilde R^2/4}x^{D/2}\cdot \frac{x^{d/2-1}e^{-x/2}}{2^{d/2}\Gamma(d/2)}\dd x,
\end{align*}
where the second inequality follows from the probability density function of $\chi_d^2$ distribution and the fact that $R\ge 1$. Moreover, assuming $d\ge D$, it is easy to verify that when $x\ge 2Dd$, we have
\begin{align*}
x^{H/2}\cdot \frac{x^{d/2-1}e^{-x/2}}{2^{d/2}\Gamma(d/2)}\le \frac{(x/2)^{d/2-1}e^{-x/4}}{2^{d/2}\Gamma(d/2)}.
\end{align*}
Thus, if $\tilde R^2\ge 8Dd/(m\beta)$, we have
\begin{align*}
\int_{x\ge m\beta \tilde R^2/4}x^{H/2}\cdot \frac{x^{d/2-1}e^{-x/2}}{2^{d/2}\Gamma(d/2)}\dd x\le \int_{x\ge m\beta \tilde R^2/4}\frac{(x/2)^{d/2-1}e^{-x/4}}{2^{d/2}\Gamma(d/2)}\dd x = 2\PP_{z\sim \chi_d^2}[z\ge m\beta \tilde R^2/8].
\end{align*}
By standard tail bound of $\chi_d^2$ distribution, we have 
\begin{align*}
\PP_{z\sim\chi_d^2}\big[z\ge d + 2 \sqrt{d\log(1/\delta)}+2\log(1/\delta)\big]\le \delta.
\end{align*}
Therefore, set $\tilde R = \big[8(d+2\sqrt{d\log(1/\epsilon)}+2\log(1/\epsilon))/(m\beta)\vee8Dd/(m\beta)]^{1/2}=\tilde O(d^{1/2})$, we have
\begin{align*}
\int_{\RR^d\bs\tilde \Omega} h(\xb) \pi(\xb)\le 4C\epsilon. 
\end{align*}
Combining all previous results, we obtain
\begin{align*}
\EE[h(\xb_K)] - \EE[h(\xb^\pi)] \le  C(5+\tilde R^D)\epsilon.
\end{align*}
Applying the fact that $\tilde R = \tilde O(d^{1/2})$, we are able to complete the proof.
\end{proof}

\subsection{Proof of Theorem \ref{thm:main_thm_hessian}}
The main body of the proof of Theorem \ref{thm:main_thm_hessian} is the same as that of Theorem \ref{thm:main_thm}. The only difference/improvement is that provided Assumption \ref{assump:hessian_lip}, a sharper approximation error between the transition distributions $\transitRef_\ub(\cdot)$ and $\transitSGLD_\ub(\cdot)$ can be proved, implying SGLD is closer to its metropolized counterpart. We formally state this result in the following lemma.
\begin{lemma}\label{lemma:approximation_gld}
Under Assumptions \ref{assump:smooth} and \ref{assump:hessian_lip}, let $G =\|\nabla f(\zero)\|_2$ and set $r = \sqrt{10\eta d/\beta}\big(1+\sqrt{\log(8K/\epsilon)/d}\big)$. Then 
there exists a constant 
\begin{align*}
\delta &= \big[28Hd^{3/2}\beta^{-1/2}\eta^{3/2} +10L(LR+G)d^{1/2}\beta^{1/2}\eta^{3/2} + 12\beta (LR+G)^2 d\eta/B+2 \beta^2(LR+G)^4\eta^2/B\big]\notag\\
&\qquad \cdot \bigg(1+\sqrt{\frac{\log(8K/\epsilon)}{d}}\bigg)^{2}
\end{align*}
such that for any set $\cA\subseteq\Omega$ and any point $\ub\in\Omega$, it holds that
\begin{align*}
(1-\delta)\transitRef_\ub(\cA)\le \transitSGLD_\ub(\cA)\le (1+\delta)\transitRef_\ub(\cA). 
\end{align*}

\end{lemma}

\begin{proof}[Proof of Theorem \ref{thm:main_thm_hessian}]
Similar to the proof of Theorem \ref{thm:main_thm}, we first characterize the feasible range of $\eta$ that satisfies all requirements in Lemmas \ref{lemma:approximation_gld}, \ref{lemma:convergence_approximate} and \ref{lemma:lowerbound_phi}. Then by Lemma \ref{lemma:approximation_gld}, we know that if $\eta\le\big[25d\beta(LR+G)^2\big]^{-1}\wedge L^2\beta d^{-1}H^{-2}/25$ and $\beta\ge1$, the transition distribution $\transitSGLD_\ub(\cdot)$ can be $\delta$-close to $\transitRef_\ub(\cdot)$ with
\begin{align*}
\delta &= \big[28Hd^{3/2}\beta^{-1/2}\eta^{3/2} +10L(LR+G)d^{1/2}\beta^{1/2}\eta^{3/2} + 12\beta (LR+G)^2 d\eta/B+2 \beta^2(LR+G)^4\eta^2/B\big]\notag\\
&\qquad \cdot \bigg(1+\sqrt{\frac{\log(8K/\epsilon)}{d}}\bigg)^{2}\notag\\
&\le \Big[14(LR+G)^2\beta d\eta/B + \big[28Hd^{3/2}\beta^{-1/2}+10L(LR+G)d^{1/2}\beta^{1/2}\big]\eta^{3/2}\Big]\cdot\bigg(1+\sqrt{\frac{\log(8K/\epsilon)}{d}}\bigg)^{2}.
\end{align*}
Then based on the requirement of $\eta$  and the lower bound of $\phi$ in Lemma \ref{lemma:lowerbound_phi}, it suffices to set the step size to be 
\begin{align*}
\eta &\le \min\bigg\{\frac{1}{25\beta(LR+G)^2},\frac{1}{35(Ld+(LR+G)^2\beta d/B)},\bigg(\frac{c_0\rho}{224(LR+G)^2\beta^{3/2} d/B}\bigg)^2,\notag\\
&\qquad \qquad\qquad\frac{c_0\rho}{16\sqrt{\beta}\big[28Hd^{3/2}\beta^{-1/2}+10L(LR+G)d^{1/2}\beta^{1/2}\big]}\bigg\}\cdot \bigg(1+\sqrt{\frac{\log(8K/\epsilon)}{d}}\bigg)^{-4}
\end{align*}
Therefore, by Lemma \ref{lemma:approximate_target}, set $R = \bar R(\epsilon/12)$, we have
\begin{align*}
\|\mu_k^{\text{SGLD}} - \pi\|_{TV}&\le \|\pi - \stationalRef\|_{TV} + \|\mu_k^{\text{Proj-SGLD}} - \stationalRef\|_{TV} + \|\mu_k^{\text{SGLD}} - \mu_k^{\text{Proj-SGLD}}\|_{TV} \notag\\
&\le \frac{\epsilon}{2}+\lambda\big(1 - \phi^2/8\big)^k + \frac{16\delta}{\phi} \notag\\
&\le \frac{\epsilon}{2} + \lambda(1-C_0\eta)^{k} + C_1B^{-1}\eta^{1/2}+C_2\eta,
\end{align*}
where $C_0 =c_0^2\rho^2/(8\beta)$, $C_1 = 224\big(1+\sqrt{\log(8K/\epsilon)/d}\big)^4(LR+G)^2\beta^{3/2} d\rho^{-1}/c_0$, $C_2 = \big(1+\sqrt{\log(8K/\epsilon)/d}\big)^4\rho^{-1}\big[448Hd^{3/2}+160L(LR+G)d^{1/2}\beta\big]/c_0$ are problem-dependent constants.
\end{proof}

\subsection{Proof of Corollary \ref{coro:sgld2}}
\begin{proof}[Proof of Corollary \ref{coro:sgld2}]
From \eqref{eq:upperbound_lambda}, we know that $\mu_0$ is a $\lambda$-warm start with respect to $\stationalRef$ with $\lambda = e^{O(d)}$. Then in order to guarantee that the sampling error $\|\mu_k^{\text{SGLD}}-\pi\|_{TV}\le \epsilon$, it suffices to set
\begin{align*}
\lambda(1-C_0\eta)^k = \epsilon/6,\ C_1B^{-1}\eta^{1/2}= \epsilon/6\ , C_2\eta = \epsilon/6.
\end{align*}
Note that we have $R = \bar R(\epsilon/12) = O\big(d^{1/2}\beta^{-1/2}\log^{1/2}(\epsilon)\big)$, $C_0 = O(\rho^2\beta^{-1})$, $C_1 = \tilde O(d^2\rho^{-1}\beta^{1/2})$ and $C_2 = \tilde O\big(d^{3/2}\rho^{-1}+d\beta^{1/2}\rho^{-1}\big)$, plugging these into the above equation gives
\begin{align*}
\eta &= \tilde O\bigg(\frac{\rho^2B^2\epsilon^2}{d^4\beta}\wedge \frac{\rho\epsilon}{d^{3/2}+d\beta^{1/2}}\bigg) \notag\\
k& = O\bigg(\frac{\log(\lambda/\epsilon)}{C_0\eta}\bigg) = \tilde O\bigg(\frac{d^5\beta^2}{\rho^4B^2\epsilon^2}+\frac{d^{5/2}\beta + d^2\beta^{3/2}}{\rho^3\epsilon}\bigg).
\end{align*}
This completes  the proof.
\end{proof}

\section{Proof of Lemmas in Section \ref{sec:proof_main}}\label{sec:proof_keylemma}
In this section, we provide the proof of Lemmas used in Section \ref{sec:proof_main}.

\subsection{Proof of Lemma \ref{lemma:connection_SGLD}}\label{sec:connection_sgld}
The idea is to show if the quantities $R$ and $r$ satisfy
\begin{align}\label{eq:choice_R_r}
R \ge \max\bigg\{25\bigg(\frac{d\log(4K/\delta)}{m\beta}\bigg)^{1/2},4\bigg(\frac{b+G^2+d\beta^{-1}}{m}\bigg)^{1/2}\bigg\},\quad r\ge \sqrt{2\eta d/\beta}(2+\sqrt{2\log(2K/\delta)/d}),
\end{align}
and $\|\xb_0\|_2\le R/2$ holds with probability at least $1-\delta/2$ for arbitrary $\delta\in(0,1)$, Algorithm \ref{alg:projected_sgld} generates the same output as that of the standard SGLD
with probability at least $1-\delta$. 

We first show why this is sufficiently to prove the upper bound on the total variation distance between $\mu_K^{\text{SGLD}}$ and $\mu_K^{\text{Proj-SGLD}}$. In particular, we have
\begin{align*}
\big\|\mu_K^{\text{SGLD}}-\mu_K^{\text{Proj-SGLD}}\big\|_{TV} =\sup_{\cA\in\RR^d}\big|\mu_K^{\text{SGLD}}(\cA) - \mu_K^{\text{Proj-SGLD}}(\cA)\big|.
\end{align*}
Let $\xb_K$ and $\hat \xb_K$ be the outputs of SGLD and projected SGLD, we can rewritten $\hat\xb_K$ as
\begin{align*}
\xb_K = \hat\xb_K\cdot\ind(\xb_k=\hat\xb_K) + \xb_K\cdot\ind(\xb_k\neq\hat\xb_K).
\end{align*}
Then for any set $\cA$, it holds that
\begin{align*}
\ind(\xb_K\in\cA) &= \ind(\hat\xb_K\in\cA)\cdot\ind(\xb_k=\hat\xb_K) + \ind(\xb_K\in\cA)\cdot\ind(\xb_k\neq\hat\xb_K)\notag\\
& = \ind(\hat\xb_K\in\cA) -\ind(\hat\xb_K\in\cA)\cdot\ind(\xb_k\neq\hat\xb_K)+ \ind(\xb_K\in\cA)\cdot\ind(\xb_k\neq\hat\xb_K).
\end{align*}
This further implies that 
\begin{align*}
\ind(\xb_K\in\cA)-\ind(\hat\xb_K\in\cA)&\le \ind(\xb_K\in\cA)\cdot\ind(\xb_k\neq\hat\xb_K)\le \ind(\xb_k\neq\hat\xb_K);\notag\\
\ind(\xb_K\in\cA)-\ind(\hat\xb_K\in\cA)&\ge -\ind(\hat\xb_K\in\cA)\cdot\ind(\hat\xb_k\neq\xb_K)\ge -\ind(\xb_k\neq\hat\xb_K).
\end{align*}
Taking expectation on both sides, we can get that 
\begin{align*}
\big|\mu_K^{\text{SGLD}}(\cA) - \mu_K^{\text{Proj-SGLD}}(\cA)\big|\le \EE[\ind(\xb_k\neq\hat\xb_k)]\le\delta,
\end{align*}
for any $\cA\in\RR^d$, and thus
\begin{align*}
\big\|\mu_K^{\text{SGLD}}-\mu_K^{\text{Proj-SGLD}}\big\|_{TV}\le \delta.
\end{align*}

The next step is to show that Projected-SGLD generates the same outputs as that of the standard SGLD with probability at least $1-\delta$, which suffices to show that with probability at least $1-\delta$, Projected-SGLD will accept all $K$ iterates. In other words, let $\{\xb_k\}_{k=0,\dots, K}$ be the iterates generated by the standard SGLD (without accept/reject step), our goal is to prove that with probability at least $1-\delta$, all $\xb_k$'s stay inside the region $\cB(\boldsymbol{0}, R)$, and $\|\xb_{k}-\xb_{k-1}\|_2\le r$ for all $k\le K$. These properties are summarized in the following two facts. 
\paragraph{Fact 1:} With probability at least $1-\delta/2$, all iterates stay inside  the region $\cB(\boldsymbol{0},R)$ 
\paragraph{Fact 2:} Given \textbf{Fact 1}, with probability at least $1-\delta/2$, $\|\xb_k - \xb_{k-1}\|_2\le r$ for all $k\le K$.

The following lemma will be useful to the proof.
\begin{lemma}[Lemma 3.1 in \citet{raginsky2017non}]\label{lemma:grad_bound} 
Under Assumption \ref{assump:smooth}, there exists a constant $G = \max_{i\in[n]} \|\nabla f_i(\zero)\|_2 $ such that for any $\xb\in\RR^d$ and $i\in[n]$, it holds that
\begin{align*}
\|\nabla f_i(\xb)\|_2\le L\|\xb\|_2 + G.
\end{align*}
\end{lemma}
Now we will proceed to proving these two facts. 

\begin{proof}[Proof of Lemma \ref{lemma:connection_SGLD}]
Regarding \textbf{Fact 1}, we first take a look at $\|\xb_k\|_2^2$. By Assumption \ref{assump:diss}, we have 
\begin{align*}
\EE[\|\xb_{k+1}\|_2^2|\xb_k] &= \EE\big[\big\|\xb_k - \eta\gb(\xb_k, \cI) + \sqrt{2\eta/\beta}\bepsilon_k\big\|_2^2|\xb_k\big]\notag\\
&= \|\xb_k\|_2^2 - 2\eta \EE[ \la \xb_k, \gb(\xb_k, \cI)\ra|\xb_k] + \eta^2\EE[\|\gb(\xb_k, \cI)\|_2^2|\xb_k] + \frac{2d\eta}{\beta}\notag\\
&\le (1-2m\eta)\|\xb_k\|_2^2 +2\eta b + \eta^2(L\|\xb_k\|_2+G)^2 + \frac{2d\eta}{\beta}\notag\\
&\le (1-2m\eta+2L^2\eta^2)\|\xb_k\|_2^2 +2\eta b + 2\eta^2G^2 + \frac{2d\eta}{\beta},
\end{align*}
where the first inequality follows from Assumption \ref{assump:diss} and Lemma \ref{lemma:grad_bound}, and the last inequality is due to Young's inequality. If we choose $\eta\le 1\wedge m/(4L^2) $, the above inequality implies that
\begin{align*}
\EE[\|\xb_{k+1}\|_2^2|\xb_k] &\le (1-3m\eta/2)\|\xb_k\|_2^2 +2\eta b + 2\eta G^2 + \frac{2d\eta}{\beta}.
\end{align*}
Then it is clear that if $\|\xb_k\|_2^2\ge (4b + 4G^2 + 4d\beta^{-1})/m$, the above inequality implies that $\EE[\|\xb_{k+1}\|_2^2|\xb_k]\le (1-m\eta)\|\xb_k\|_2^2$. Note that in order to prove $\|\xb_k\|_2\le R$ for all $k\le K$, we only need to consider $\xb_k$ satisfying $\|\xb_k\|_2^2\ge (4b + 4G^2 + 4d\beta^{-1})/m$ since our choice of $R$ satisfies $R\ge 2\sqrt{(4b + 4G^2 + 4d\beta^{-1})/m}$ (see \eqref{eq:choice_R_r}), otherwise $\|\xb_k\|_2\le \sqrt{(4b + 4G^2 + 4d\beta^{-1})/m}\le R$ naturally holds. Then by the concavity of the function $\log(\cdot)$, for any $\|\xb_k\|_2\ge R/2$, we have
\begin{align}\label{eq:martingale_super}
\EE[\log(\|\xb_{k+1}\|_2^2)|\xb_k] \le \log\big(\EE[\|\xb_{k+1}\|_2^2)|\xb_k]\big) \le \log(1-m\eta) + \log(\|\xb_k\|_2^2)\le  \log(\|\xb_k\|_2^2) - m\eta . 
\end{align}
Besides, by triangle inequality it holds that
\begin{align*}
\|\xb_{k+1}\|_2 - \|\xb_{k}\|_2\le \eta\|\gb(\xb_k,\cI)\|_2 + \sqrt{\eta/\beta}\|\bepsilon_k\|_2.
\end{align*}
Note that $\|\bepsilon_k\|_2$ is the square root of a $\chi(d)$ random variable, which is subgaussian and satisfies $\PP(\|\bepsilon_k\|_2\ge \sqrt{d}+\sqrt{2}t)\le e^{-t^2}$ for all $z\ge 0$. Besides, if $\|\xb_k\|_2\le R$, by Lemma \ref{lemma:grad_bound} we have $\|\gb(\xb_k,\cI)\|_2\le LR+G$, then assume $\eta\le d(LR+G)^{-1}/\beta$, this further implies that  
\begin{align*}
\PP\big(\|\xb_{k+1}\|_2 - \|\xb_{k}\|_2\ge2\sqrt{\eta d/\beta} + \sqrt{2\eta/\beta} t \big)\le e^{-t^2}.
\end{align*}
for all $z\ge 0$.
Further note that we assumed $\|\xb_k\|_2\ge R/2$, it follows that
\begin{align*}
\log(\|\xb_{k+1}\|_2^2) - \log(\|\xb_{k}\|_2^2) = 2\log(\|\xb_{k+1}\|_2/\|\xb_{k}\|_2)\le \|\xb_{k+1}\|_2/\|\xb_{k}\|_2 -1 \le \frac{2\|\xb_{k+1}\|_2 - 2\|\xb_k\|_2}{R}.
\end{align*}
Therefore we have $\log(\|\xb_{k+1}\|_2^2) - \log(\|\xb_{k}\|_2^2)$ is also a subgaussian random variable and satisfies
\begin{align}\label{eq:martingale_difference} 
\PP\big(\log(\|\xb_{k+1}\|_2^2) - \log(\|\xb_{k}\|_2^2)\ge 4\sqrt{\eta d/\beta}R^{-1} +2t\sqrt{2\eta/\beta}R^{-1}\big)\le \exp(-t^2).
\end{align}
We will consider any
subsequence among $\{\xb_k\}_{k=1,\dots,K}$, with all iterates, except the first one, staying outside the region $\cB(\boldsymbol{0}, R/2)$. Denote such subsequence by $\{\yb_{k}\}_{k=0,\dots, K'}$, where $\yb_0$ satisfies  $\|\yb_0\|_2\le R/2$ and $K'\le K$. Then it can be seen that $\yb_k, \yb_{k+1}$ satisfy \eqref{eq:martingale_super} and \eqref{eq:martingale_difference} for all $k\ge 1$. Further note that by our assumption on the initialization $\mu_0$, with probability at least $1-\delta/4$ we have the initial point $\xb_0$ satisfies $\|\xb_0\|_2\le R/2$. Then it suffices to prove that with probability at least $1-\delta/4$ all points in $\{\yb_{k}\}_{k=1,\dots, K'}$ will stay inside the region $\cB(\boldsymbol{0}, R)$. 
Then let $\cE_k$ be the event that $\|\yb_{k'}\|_2\le R$ for all $k'\le k$, and $\cF_k = \{\yb_0,\dots,\yb_k\}$ be the filtration, it is easy to see that $\cE_k\subseteq\cE_{k-1}$ and thus the sequence $\{\ind(\cE_{k-1})\cdot (\log(\|\yb_{k}\|_2^2)+km\eta)|\cF_{k-1}\}_{k=1,\dots,K'}$  is a  super-martingale. Besides, we can show that the martingale difference has a subgaussian tail, i.e., for any $t\ge 0$, 
\begin{align*}
\mathbb P\big(\|\yb_{k+1}\|_2^2+(k+1)m\eta - \log(\|\yb_{k}\|_2^2)-km\eta\ge 5\sqrt{\eta d/\beta}R^{-1}+2t\sqrt{2\eta d/\beta}R^{-1}\big) \le \exp(-t^2),
\end{align*}
where we use the fact that $\eta\le d\beta^{-1}m^{-2}$ and $d\ge 1$. Let $z = 5\sqrt{\eta d/\beta}R^{-1}+2t\sqrt{2\eta d/\beta}R^{-1}$, we have $t = \sqrt{\beta/(8\eta d)}\cdot\big(Rz-5\sqrt{\eta d/\beta}\big)$. Then note that $(a-b)^2\ge \frac{a^2}{4}-b^2/3$ for all $a,b\in\RR$, we have
\begin{align*}
t^2 \ge \frac{\beta R^2 z^2}{32\eta d} - \frac{25}{24}.
\end{align*}
Therefore, for any $z\ge 0$ we have 
\begin{align*}
\mathbb P\big(\|\yb_{k+1}\|_2^2+(k+1)m\eta - \log(\|\yb_{k}\|_2^2)-km\eta\ge z\big) \le \exp\bigg(-\frac{\beta R^2 z^2}{32\eta d} + \frac{25}{24}\bigg)\le 3\exp\bigg(-\frac{\beta R^2 z^2}{32\eta d}\bigg),
\end{align*}
which implies that the martingale difference is subgaussian.
Then by Theorem 2 in \citet{shamir2011variant} (one-side), we have for a given $k$, conditioned on the event $\cE_{k-1}$, with probability at least $1-\delta'$, 
\begin{align*}
\log(\|\yb_{k}\|_2^2) + km\eta\le \log(\|\yb_{0}\|_2^2) + 52\sqrt{k\eta d\log(1/\delta')/(\beta R^2)}
\end{align*}
for some absolute positive constant $R$.
Taking union bound over all $k=1,\dots,K'$ ($K'\le K$) and defining $ \delta=4\delta'K'$, we have with probability at least $1-\delta/4$, for all $k=1,\dots,K'$ it holds that 
\begin{align*}
\log(\|\yb_{k}\|_2^2) &\le 2\log(R/2) +  52\sqrt{k\eta d\log(2K/\delta)/(\beta R^2)} - m k\eta\notag\\
&\le 2\log(R/2) + \frac{700 d\log(2K/\delta)}{m\beta R^2}.
\end{align*}
Applying our choice of $R$ in \eqref{eq:choice_R_r} gives
\begin{align*}
\frac{700 d\log(2K/\delta)}{m\beta R^2} \le 2\log(2).
\end{align*}
Therefore, for all $k=1,\ldots,K$, we have with probability at least $1-\delta/4$ that
\begin{align*}
\log(\|\yb_{k}\|_2^2)\le 2\log(R/2)  + 2\log(2) = \log(R^2),
\end{align*}
which is equivalent to $\|\yb_k\|_2\le R$. Combining with the fact that with probability at least $1-\delta/4$ the initial point $\xb_0$ stays inside $\cB(0,R/2)$, we can conclude that with probability at least $1-\delta/2$,  all iterates stay inside the region $\cB(\boldsymbol{0}, R)$, which completes the proof of \textbf{Fact 1}.

Now we proceed to prove \textbf{Fact 2}, of which  the key is to prove  $\|\xb_{k}-\xb_{k-1}\|_2\le r$ for all $k\ge K$. Note that in each iteration, the proposal distribution of $\xb_{k+1}$ is an expected Gaussian distribution. Besides, note that for all possible mini-batch, the drift term satisfies
\begin{align*}
\eta\|\gb(\xb_k,\cI)\|_2\le \eta (LR + G).
\end{align*}
This implies that the probability that $\xb_{k+1}\notin \cB(\xb_{k},r)$ can be upper bounded by
\begin{align*}
\PP[\xb_{k+1}\notin \cB(\xb_{k},r)]\le \PP_{z\sim \chi_d^2}\big[2\eta \beta^{-1} z\le (r -\eta (LR + G))^2\big] = \PP_{z\sim \chi_d^2}\big[\sqrt{z}\le [r - \eta(LR+G)]/(2\eta\beta^{-1})^{1/2}\big].
\end{align*}
Note that $\eta\le (LR+G)^{-2}\beta^{-1}d$, we have $\eta(LR+G)/(2\eta\beta^{-1})^{1/2}\le d^{1/2}$. Then by standard tail bound of Chi-square distribution and our choice of $r$ in \eqref{eq:choice_R_r},
\begin{align*}
\PP[\xb_{k+1}\notin \cB(\xb_{k},r)]\le\PP_{z\sim \chi_d^2}\big[\sqrt{z}\le d^{1/2}+\sqrt{2\log(2K/\delta)}\big]\le 1 - \delta/(2K).
\end{align*}
Taking union bound over all iterates, we are able to complete the proof of \textbf{Fact 2}.

Combining \textbf{Fact 1} and \textbf{Fact 2} and set $\delta = \epsilon/4$ complete the proof of Lemma \ref{lemma:connection_SGLD}. 
\end{proof}

\subsection{Proof of Lemma \ref{lemma:approximation}}
Before providing the detailed proof of Lemma \ref{lemma:approximation}, we first present the following useful lemma.

\begin{lemma}\label{lemma:subgaussian_without}
Let $\gb(\xb,\cI)$ be the stochastic gradient with mini-batch size $|\cI|= B<n$, then for any vector $\ab$ and $\|\xb\|_2\le R$, there exists a constant $M =LR+G$ such that
\begin{align*}
\EE_{\cI}\big[\exp\big(\la \ab,\gb(\xb,\cI)-\nabla f(\xb)\ra\big)\big] \le \exp(M^2\|\ab\|_2^2/B).
\end{align*}
Moreover, we have $\EE_{\cI}\big[\exp\big(\la \ab,\gb(\xb,\cI)-\nabla f(\xb)\ra\big)\big]=1$ if $B = n$.
\end{lemma}
\begin{proof}[Proof of Lemma \ref{lemma:approximation}]
Note that the Markov processes defined by $\transitRef_\ub(\cdot)$ and $\transitSGLD_\ub(\cdot)$ are $1/2$-lazy according to \eqref{eq:def_trans_lz_sgld}. We prove the lemma by considering two cases: $\ub\not\in\cA$ and $\ub\in\cA$. 
We first prove the lemma in the first case. Note that when $\ub\notin\cA$, we have
\begin{align}\label{eq:relation_Q2P_case1}
    \transitRef_{\ub}(\cA)=\int_{\cA} \transitRef_\ub(\wb)\dd\wb = \int_{\cA} \alpha_\ub(\wb) \transitSGLD_\ub(\wb)\dd\wb.
\end{align}
By \eqref{eq:markov_chain_v2w}, we know that $\wb$ is restricted in $\wb\in \cB(\ub,r)\cap\cB(\zero,R)\bs\{\ub\}$. For sufficiently small step size $\eta$, we can ensure $\delta\le 1/2$. In the rest of this proof we will show that $\alpha_\ub(\wb)\ge 1-\delta/2$ for all $\wb\in \cB(\ub,r)\cap\cB(\zero,R)\bs\{\ub\}$, which together with \eqref{eq:relation_Q2P_case1} implies
\begin{align*}
(1-\delta/2)\transitSGLD_\ub(\cA)\le \transitRef_\ub(\cA)\le \transitSGLD_\ub(\cA), 
\end{align*}
and thus \eqref{eq:delta_approximation} also holds since $\alpha_\ub(\wb)\ge 1-\delta/2$.
Then, it suffices to prove that 
\begin{align}\label{eq:lowerbound_accept}
\frac{\transitSGLD_{\wb}(\ub)}{\transitSGLD_{\ub}(\wb)}\cdot\exp(-\beta(f(\wb) - f(\ub)))\ge 1-\delta/2.
\end{align}
By the definition of $\transitSGLD_\ub(\wb)$, it is equivalent to proving 
\begin{align*}
\frac{\EE_{\cI_1} \bigg[\exp\bigg(-\beta f(\wb) - \frac{\|\ub - \wb + \eta\gb(\wb,\cI_1)\|_2^2}{4\eta/\beta}\bigg)\bigg|\wb\bigg]}{\EE_{\cI_2} \bigg[\exp\bigg(-\beta f(\ub) - \frac{\|\wb - \ub + \eta\gb(\ub,\cI_2)\|_2^2}{4\eta/\beta}\bigg)\bigg|\ub\bigg]}\ge 1- \delta/2,
\end{align*}
where $\cI_1,\cI_2\subseteq[n]$ are two independent mini-batches of data. 
Let $I_1$ and $I_2$ denote the numerator and denominator of the L.H.S. of the above inequality respectively. Then regarding $I_1$, by Jensen's inequality and convexity of the function $\exp(\cdot)$, we have 
\begin{align}\label{eq:bound_I1_1}
I_1&\ge  \exp\bigg(-\beta f(\wb) - \frac{\EE_{\cI_1}[\|\ub - \wb + \eta\gb(\wb,\cI_1)\|_2^2|\wb]}{4\eta/\beta}\bigg) \notag\\
&= \exp\bigg(-\beta f(\wb) - \frac{\|\ub - \wb\|_2^2  + 2\eta\la\ub - \wb, \nabla f(\wb)\ra + \eta^2\EE_{\cI_1}[\|\gb(\wb,\cI_1)\|_2^2|\wb]}{4\eta/\beta}\bigg)\notag\\
&\ge \exp\bigg(-\beta f(\wb) - \frac{\|\ub - \wb\|_2^2  + 2\eta\la\ub - \wb, \nabla f(\wb)\ra + \eta^2\|\nabla f(\wb)\|_2^2+\eta^2M^2d/B}{4\eta/\beta}\bigg),
\end{align}
where the last inequality is by Lemma  \ref{lemma:subgaussian_without}.
Then we move on to upper bounding $I_2$, 
\begin{align*}
I_2 &= \exp\big(-\beta f(\ub)\big) \cdot\EE_{\cI_2} \bigg[\exp\bigg(-\frac{\|\wb - \ub\|_2^2  + 2\eta\la\wb - \ub, \gb(\wb,\cI_2)\ra + \eta^2\|\gb(\ub,\cI_2)\|_2^2}{4\eta/\beta}\bigg)\bigg|\ub\bigg]\notag\\
&=\exp\bigg(-\beta f(\ub) - \frac{\|\wb-\ub\|_2^2+2\eta\la\wb - \ub,\nabla f(\ub)\ra }{4\eta/\beta}\bigg)\notag\\
&\qquad\cdot \underbrace{\EE_{\cI_2} \bigg[\exp\bigg(-\frac{2\beta\la\wb - \ub, \gb(\ub,\cI_2)-\nabla f(\ub)\ra + \beta\eta\|\gb(\ub,\cI_2)\|_2^2}{4}\bigg)\bigg|\ub\bigg]}_{I_3}.
\end{align*}
Note that 
\begin{align*}
\|\gb(\ub,\cI_2)\|_2^2 &= \|\nabla f(\ub)\|_2^2 + 2\la\nabla f(\ub), \gb(\ub,\cI_2)-\nabla f(\ub)\ra + \|\gb(\ub,\cI_2)-\nabla f(\ub)\|_2^2\notag\\
&\ge \|\nabla f(\ub)\|_2^2 + 2\la\nabla f(\ub), \gb(\ub,\cI_2)-\nabla f(\ub)\ra.
\end{align*}
Then we have the following regarding $I_3$,
\begin{align*}
I_3 &\le \EE_{\cI_2} \bigg[\exp\bigg(-\frac{2\beta\la\wb - \ub+\eta\nabla f(\ub), \gb(\ub,\cI_2)-\nabla f(\ub)\ra}{4}\bigg)\bigg|\ub\bigg]\cdot\exp\bigg(-\frac{\beta\eta\|\nabla f(\ub)\|_2^2}{4}\bigg)\notag\\
&\le \exp\bigg(\frac{\beta^2M^2\|\wb-\ub+\eta\nabla f(\ub)\|_2^2}{4B}\bigg)\cdot\exp\bigg(-\frac{\beta\eta\|\nabla f(\ub)\|_2^2}{4}\bigg),
\end{align*}
where the last inequality holds due to Lemma \ref{lemma:subgaussian_without}.
Then by Young's inequality, $I_3$ can be further upper bounded by
\begin{align*}
I_3&\le \exp\bigg(\frac{\beta^2M^2\big(\|\wb-\ub\|_2^2+\eta^2\|\nabla f(\ub)\|_2^2\big)}{2B}-\frac{\beta\eta\|\nabla f(\ub)\|_2^2}{4}\bigg)\notag\\
&\le \exp\bigg(\frac{\beta^2 M^2\big(\|\wb-\ub\|_2^2+(LR+G)^2\eta^2\big)}{2B}-\frac{\beta\eta\|\nabla f(\ub)\|_2^2}{4}\bigg),
\end{align*}
where the second inequality is by Lemma \ref{lemma:grad_bound}.
Combining the previous results for $I_1$ and $I_2$, we have
\begin{align}\label{eq:lower_bound_I1/I2}
\frac{I_1}{I_2} &\ge \exp\bigg(-\beta\big(f(\wb) - f(\ub)\big) - \frac{\beta\la\ub - \wb, \nabla f(\wb) + f(\ub)\ra}{2}\bigg)\notag\\
&\qquad\cdot \exp\bigg(\frac{\beta\eta\big(\|\nabla f(\ub)\|_2^2 -\|\nabla f(\wb)\|_2^2 \big)}{4}-\frac{\beta^2 M^2\big(\|\wb-\ub\|_2^2+(LR+G)^2\eta^2\big)+\beta \eta M^2d/2}{2B}\bigg).
\end{align}
It is well known  that the smoothness condition in Assumption \ref{assump:smooth} \citep{nesterov2018lectures} is equivalent to the following inequalities,
\begin{align*}
    f(\wb) &\leq f(\ub) +\la\wb - \ub, \nabla f(\ub)\ra+\frac{L\|\wb -\ub\|_2^2}{2},\\
    f(\ub) &\geq f(\wb) +\la\ub - \wb, \nabla f(\wb)\ra-\frac{L\|\wb - \ub\|_2^2}{2},
\end{align*}
which immediately implies
\begin{align}
\big|f(\wb) - f(\ub) -\la\wb - \ub, \nabla f(\wb) + f(\ub)\ra/2\big|\le \frac{L\|\wb - \ub\|_2^2}{2}.\label{eq:Gu0001}
\end{align}
In addition, by Lemma \ref{lemma:grad_bound}  
and Assumption \ref{assump:smooth}, it holds that
\begin{align}\label{eq:upperbound_variance_difference}
\big|\|\nabla f(\ub)\|_2^2 - \|\nabla f(\wb)\|_2^2\big|&\leq \|\nabla f(\ub)-\nabla f(\wb)\|_2\cdot\|\nabla f(\ub)+\nabla f(\wb)\|_2\notag\\
&\le 2L(LR+G)\|\wb-\ub\|_2.
\end{align}
Now, substituting \eqref{eq:Gu0001} and \eqref{eq:upperbound_variance_difference} into \eqref{eq:lower_bound_I1/I2} and using the fact that $\|\wb-\ub\|_2\le r = \sqrt{10\eta d/\beta}\big(1+\sqrt{\log(8K/\epsilon)/d}\big)$, we have
\begin{align*}
\frac{I_1}{I_2}&\ge \exp\bigg(-\frac{L\beta\|\wb - \ub\|_2^2}{2} -\frac{\beta \eta L(LR+G)\|\wb-\ub\|_2 }{2} \notag\\
&\hspace{30mm}-\frac{\beta^2 M^2\big(\|\wb-\ub\|_2^2+(LR+G)^2\eta^2+\eta d\beta^{-1}/2\big)}{2B}\bigg)\notag\\
&\ge \exp\bigg[-\bigg(5Ld\eta +5L(LR+G)d^{1/2}\beta^{1/2}\eta^{3/2} +\frac{6\beta M^2 d\eta}{B}+ \frac{\beta^2M^2(LR+G)^2\eta^2}{2B}\bigg)\notag\\
&\hspace{20mm}\cdot\bigg(1+\sqrt{\frac{\log(8K/\epsilon)}{d}}\bigg)^2\bigg]\notag\\
&\ge 1 -\big[5Ld\eta  -5L(LR+G)d^{1/2}\beta^{-1/2}\eta^{3/2} - 6\beta M^2 d\eta/B- \beta^2M^2(LR+G)^2\eta^2/B\big]\notag\\
&\qquad\cdot \bigg(1+\sqrt{\frac{\log(8K/\epsilon)}{d}}\bigg)^2 \notag\\
&= 1- \delta/2,
\end{align*}
where we plug in the fact that $K=LR+G$ in the last equality. This completes the proof for the case $\ub\notin\cA$. 

In the second case that $\ub\in\cA$, we can split $\cA$ into $\{\ub\}$ and $\cA\bs\{\ub\}$. Note that by our result in the first case, we have $(1-\delta)\transitRef_\ub(\cA\bs\{\ub\})\le \transitSGLD_\ub(\cA\bs\{\ub\})\le (1+\delta)\transitRef_\ub(\cA\bs\{\ub\})$. Therefore, it remains to prove that $(1-\delta)\transitRef_\ub(\ub)\le \transitSGLD_\ub(\ub)\le (1+\delta)\transitRef_\ub(\ub)$. Note that starting from $\ub$, the probability of the Markov chain generated by $\transitRef$ stays at $\ub$ is
\begin{align*}
\transitRef_\ub(\ub) = \transitSGLD_\ub(\ub) + (1-\transitSGLD_\ub(\ub))\cdot\big(1-\EE_{\wb\sim \transitSGLD_\ub(\wb|\wb\neq \ub)}[\alpha_\ub(\wb)|\ub]\big).
\end{align*}
By our previous results, we know that $\alpha_\ub(\wb)\ge 1-\delta/2$ for all $\wb\in \cB(\ub,r)\cap\Omega\bs\{\ub\}$. Therefore, we have
\begin{align*}
\transitSGLD_\ub(\ub)\le \transitRef_\ub(\ub)\le \transitSGLD_\ub(\ub) (1+ \delta/2),
\end{align*}
where the inequality on the right-hand side of $\transitRef_{\ub}(\ub)$ is due to the fact that $\transitSGLD_\ub(\ub)\ge 1/2$ and thus $\transitSGLD_\ub(\ub)\ge (1-\transitSGLD_\ub(\ub))$. Then it is evident that we have $(1-\delta)\transitRef_\ub(\ub)\le \transitSGLD_\ub(\ub)\le (1+\delta)\transitRef_\ub(\ub)$, which completes the proof for the second case.
\end{proof}

\subsection{Proof of Lemma \ref{lemma:convergence_approximate}}
Now we characterize the convergence of the projected SGLD to the truncated target distribution $\stationalRef$.
Note that Markov chains defined by $\transitSGLD_\ub(\cdot)$ and $\transitRef_\ub(\cdot)$ are restricted in the set $\Omega = \cB(\zero,R)$. Here we slightly abuse the notation by denoting $\mu_k$ be the distribution of the iterate $\xb_k$ of Project SGLD (Algorithm \ref{alg:projected_sgld}). Then define the following function
\begin{align*}
h_k(p) = \sup_{\cA:\cA\subseteq\Omega,\stationalRef(\cA)=p} \mu_k(\cA) - \stationalRef(\cA), \quad\forall p\in[0,1].
\end{align*}
Based on definition of the total variation distance between $\mu_k$ and $\pi^*$, we have
\begin{align*}
\|\mu_k-\pi^*\|_{TV} &=\sup_{\cA:\cA\subseteq\Omega} |\mu_k(\cA) - \stationalRef(\cA)|\notag\\
& = \sup_{\cA:\cA\subseteq\Omega} \max\big\{\mu_k(\cA) - \stationalRef(\cA),\stationalRef(\cA)-\mu_k(\cA)\big\}\notag\\
& = \sup_{\cA:\cA\subseteq\Omega} \max\big\{\mu_k(\cA) - \stationalRef(\cA),\mu_k(\Omega\bs\cA)-\stationalRef(\Omega\bs\cA)\big\}\notag\\
&=\sup_{\cA:\cA\subseteq\Omega} \mu_k(\cA) - \stationalRef(\cA).
\end{align*}
Then in order to prove the result in Lemma \ref{lemma:convergence_approximate}, it suffices to show that
\begin{align*}
h_k(p) \le \lambda\big(1 - \phi^2/8\big)^k + \frac{16\delta}{\phi},   
\end{align*}
for all $p\in(0,1)$. 
\begin{lemma}\label{lemma:contraction}
Let $\transitRef_\ub(\cdot)$ be a time-reversible Markov chain with unique stationary distribution $\stationalRef(\cdot)$. Then for any approximate Markov chain $\transitSGLD_\ub(\cdot)$ satisfying $(1-\delta)\transitRef_\ub(\cdot)\le \transitSGLD_\ub(\cdot)\le (1+\delta)\transitRef_\ub(\cdot)$ with $\delta\le \min\{1-\sqrt{2}/2,\phi/16\}$, there exist three parameters $\phi_k$, $\tilde\phi_k$ and $\hat \phi_k$ depending on $\mu_k(\cdot)$ that satisfy $\phi_k\ge\phi$,
\begin{align*}
2(1-\delta)\phi_k\le\tilde\phi_k\le\hat\phi_k\le 2(1+\delta)\phi_k,\text{ and }
\sqrt{1-\tilde\phi_k} + \sqrt{1+\hat\phi_k}\le 2(1-\phi_k^2/8),
\end{align*}
such that the following inequality holds  for all $p\in(0,1)$,
\begin{align*}
h_k(p)
&\le  \frac{1}{2}h_{k-1}\big(p-\tilde\phi_k\Gamma_{p}\big) + \frac{1}{2}h_{k-1}\big(p+ \hat\phi_k\Gamma_{p}\big) + 2\delta\phi_k\sqrt{\Gamma_{p}},
\end{align*}
where $\Gamma_{p}=\min\{p, 1-p\}$.
\end{lemma}

\begin{proof}[Proof of Lemma \ref{lemma:convergence_approximate}]
By Lemma \ref{lemma:contraction}, we know that if $\delta\le\min\{1-\sqrt{2}/2, \phi/16\}$, there exist three parameters $\phi_k$, $\tilde\phi_k$ and $\hat \phi_k$ depending on $\mu_k(\cdot)$ 
such that when $p\in(0, 1/2]$,
\begin{align}\label{eq:recursive}
h_k(p)
&\le  \frac{1}{2}\big[h_{k-1}\big(p-\tilde\phi_kp\big) + h_{k-1}\big(p +\hat\phi_kp\big)\big] + 2\delta\phi_k\sqrt{p},
\end{align}
and when $p\in(1/2,1)$,
\begin{align*}
h_k(p)
&\le  \frac{1}{2}\big[h_{k-1}\big(p-\tilde\phi_k(1-p)\big) + h_{k-1}\big(p +\hat\phi_k(1-p)\big)\big] + 2\delta\phi_k\sqrt{1-p}.
\end{align*}
Then, we will prove the desired result via mathematical induction. Instead of directly proving the inequality in this lemma, we aim to prove a stronger version,
\begin{align}\label{eq:bound_hk_case1}
h_k(p) \le \min\big\{\sqrt{p}, \sqrt{1-p}\big\}\cdot\bigg[\lambda\big(1 - \phi^2/8\big)^k+ \frac{16\delta}{\phi}\bigg] , 
\end{align}
We first verify the hypothesis \eqref{eq:bound_hk_case1} for the case $k=0$, based on the definition of $h_0(p)$ we have that there exists a set $\cA_0\subseteq\Omega$ satisfying $\stationalRef(\cA_0) = p$ such that
\begin{align*}
h_0(p) = \mu_0(\cA_0) - \stationalRef(\cA_0).
\end{align*}
When $p\le 1/2$, by the definition of  $\lambda$-warm initialization in \eqref{eq:def_warm_start} , it holds that 
\begin{align*}
h_0(p)&\leq\max\{\mu_0(\cA_0) - \stationalRef(\cA_0),\stationalRef(\cA_0) - \mu_0(\cA_0)\}\le \lambda p
\le \sqrt{p}\cdot\bigg(\lambda + \frac{16\delta}{\phi}\bigg).
\end{align*}
When $p\ge 1/2$, similarly we have
\begin{align*}
h_0(p)&=(1-\mu_0(\Omega\bs\cA_0)) -(1- \stationalRef(\Omega\bs\cA_0))\\
&\le  \lambda(1-p)\\
&\le \sqrt{1-p}\cdot\bigg(\lambda + \frac{16\delta}{\phi}\bigg),
\end{align*}
which verifies the hypothesis for the case $k=0$. Now we assume the hypothesis \eqref{eq:bound_hk_case1} holds for $0,\dots,k-1$. According to Lemma \ref{lemma:contraction}, the following holds when $p\in(0, 1/2]$,
\begin{align*}
h_k(p)
&\le  \frac{1}{2}\big[h_{k-1}\big(p-\tilde\phi_kp\big) + h_{k-1}\big(p+\hat\phi_kp\big)\big] + 2\delta\phi_k\sqrt{p} \notag\\
& \le  \frac{\sqrt{p-\tilde\phi_kp} + \sqrt{p+\hat\phi_kp}}{2}\bigg(\lambda(1-\phi^2/8)^{k-1} + \frac{16\delta}{\phi}\bigg) + 2\delta\phi_k\sqrt{p}\notag\\
&= \frac{\sqrt{p}\Big(\sqrt{1-\tilde\phi_k} + \sqrt{1+\hat\phi_k}\Big)}{2}\bigg(\lambda(1-\phi^2/8)^{k-1} + \frac{16\delta}{\phi}\bigg)+ 2\delta\phi_k\sqrt{p},
\end{align*}
where the second inequality is based on the hypothesis for $k-1$.
Again from Lemma \ref{lemma:contraction}, we know that $\sqrt{1-\tilde\phi_k} + \sqrt{1+\hat\phi_k}\le 2(1-\phi_k^2/8)$, which further implies
\begin{align*}
h_k(p)&\le \sqrt{p}\cdot\big(1-\phi_k^2/8\big)\bigg(\lambda(1-\phi^2/8)^{k-1} + \frac{16\delta}{\phi}\bigg)+ 2\delta\phi_k\sqrt{p}\notag\\
& \leq\sqrt{p}\cdot\bigg( \lambda(1-\phi^2/8)^{k} + \frac{16\delta}{\phi}-2\delta \phi_k^2/\phi+2\delta\phi_k\bigg)\notag\\
&\le \sqrt{p}\cdot\bigg( \lambda(1-\phi^2/8)^{k} + \frac{16\delta}{\phi}\bigg) ,
\end{align*}
where the last inequality is due to $\phi\le \phi_k$. Similar result can be proved when $p\in(1/2,1)$ and thus we omit it here. Thus we are able to verify the hypothesis for $k$. 
\end{proof}

\subsection{Proof of Lemma \ref{lemma:lowerbound_phi}}
To prove a lower bound of the conductance of $\transitRef_\ub(\cdot)$, we follow the same idea used in \citet{lee2018convergence,mangoubi2019nonconvex}, which is built upon the following lemma. 

\begin{lemma}[Lemma 13 in \citet{lee2018convergence}]\label{lemma:lowerbound_transitionprob}
Let $\transitRef_\ub(\cdot)$ be a time-reversible Markov chain on $\Omega$ with stationary distribution $\stationalRef$. Fix any $\Delta> 0$, suppose for any $\ub,\vb\in\Omega$ with $\|\ub-\vb\|_2\le \Delta$ we have $\|\transitRef_\ub(\cdot) - \transitRef_\vb(\cdot)\|_{TV}\le 0.99$, then the conductance of $\transitRef_\ub(\cdot)$ satisfies $\phi\ge C\rho\Delta$ for some absolute constant $C$, where $\rho$ is the Cheeger constant of $\stationalRef$.
\end{lemma}

Similar results have been shown in \citet{dwivedi2018log,ma2018sampling} for bounding the  $s$-conductance of Markov chains. In order to apply Lemma \ref{lemma:lowerbound_transitionprob}, we need to verify the corresponding conditions, i.e., proving that as long as $\|\ub - \vb\|_2\le \Delta$ we have $\|\transitRef_\ub(\cdot) - \transitRef_\vb(\cdot)\|_{TV}\le 0.99$ for some $\Delta$. Before moving on to the detailed proof, we first recall some definitions. Recalling \eqref{eq:trans_SGLD}, we define
\begin{align*}
P(\zb|\ub) = \EE_{\cI}[P(\zb|\ub,\cI)] = \EE_{\cI} \bigg[\frac{1}{(4\pi\eta/\beta)^{d/2}}\exp\bigg(-\frac{\|\zb - \ub + \eta\gb(\ub,\cI)\|_2^2}{4\eta/\beta}\bigg)\bigg|\ub\bigg]
\end{align*}
as the distribution after one-step standard SGLD step (i.e., without the accept/reject step).
Note that Algorithm \ref{alg:projected_sgld} only accepts the candidate iterate in the region $\Omega\cap \cB(\ub,r)$, we can compute the acceptance probability as follows,
\begin{align*}
p(\ub) = \PP_{\zb\sim P(\cdot|\ub)}\big[\zb\in\Omega\cap \cB(\ub,r)\big].
\end{align*}
Therefore, for any $\zb\in\Omega\cap \cB(\ub,r)$, the transition probability $\transitRef_{\ub}(\zb)$ takes form
\begin{align*}
\transitRef_\ub(\zb) = \frac{2 - p(\ub) + p(\ub)(1-\alpha_\ub(\zb))}{2} \delta_\ub(\zb) + \frac{\alpha_\ub(\zb)}{2}P(\zb|\ub)\cdot\ind[\zb\in\Omega\cap \cB(\ub,r)].
\end{align*}
Then the rest proof will be proving the upper bound of $\|\transitRef_\ub(\cdot) - \transitRef_\vb(\cdot)\|_{TV}$, and we state another two useful lemmas as follows.
\begin{lemma}\label{lemma:lowerboud_accept_prob_alg}
If the step size satisfies $\eta\le [40d^{-1}(LR+G)^2\beta]^{-1}$, for any $\ub\in\Omega$, the acceptance probability $p(\ub)$ satisfies $p(\ub)\ge 0.4$.
\end{lemma}

\begin{lemma}\label{lemma:TV_expected_gaussian}
Under Assumption \ref{assump:smooth}, for any two points $\ub,\vb\in\RR^d$, it holds that 
\begin{align*}
\|P(\cdot|\ub) - P(\cdot|\vb)\|_{TV}\le \frac{(1+L\eta)\|\ub-\vb\|_2}{\sqrt{2\eta/\beta}}.
\end{align*}
\end{lemma}
Lemma \ref{lemma:lowerboud_accept_prob_alg} gives a lower bound of the probability $p(\ub)$ and Lemma \ref{lemma:TV_expected_gaussian} provides an upper bound of the total variation distance between the distributions $P(\cdot|\ub)$ and $P(\cdot|\vb)$. Then we are ready to complete the proof of Lemma \ref{lemma:lowerbound_phi} as follows.

\begin{proof}[Proof of Lemma \ref{lemma:lowerbound_phi}]
Let $\cS_\ub = \Omega\cap\cB(\ub,r)$ and $\cS_\vb = \Omega\cap\cB(\vb,r)$, by triangle inequality and the definition of total variation distance, we have there exists a set $\cA\in\Omega$ such that
\begin{align*}
\|\transitRef_\ub(\cdot) - \transitRef_\vb(\cdot)\|_{TV} &= |\transitRef_\ub(\cA) - \transitRef_\vb(\cA)|\notag\\
&\le  \underbrace{\max_{\ub,\zb} \bigg[\frac{2 - p(\ub) + p(\ub)(1-\alpha_\ub(\zb))}{2}\bigg]}_{I_1} \notag\\
\qquad &+ \frac{1}{2}\underbrace{\bigg|\int_{\zb\in\cA} \alpha_\ub(\zb)P(\zb|\ub)\ind(\zb\in\cS_\ub) - \alpha_\vb(\zb)P(\zb|\vb)\ind(\zb\in\cS_\vb)\dd\zb\bigg|}_{I_2}.
\end{align*}
Then we aim to upper bound the quantities $I_1$ and $I_2$ separately. 
In  terms of $I_1$, Lemma \ref{lemma:approximation} combined with \eqref{eq:lowerbound_accept} implies that
\begin{align}\label{eq:lowerbound_accept2}
\max_{\ub,\zb}\alpha_\ub(\zb)\ge 1-\delta/2,
\end{align}
where $\delta$ is the approximation factor between $\transitSGLD_\ub(\cdot)$ and $\transitRef_\ub(\cdot)$ defined in Lemma \ref{lemma:approximation}. 
By Lemma~\ref{lemma:lowerboud_accept_prob_alg}, we know that $p(\ub)\ge 0.4$ for any $\cI$ and $\ub\in\cK$. Then combining with \eqref{eq:lowerbound_accept2}, $I_1$ can be upper bounded by
\begin{align*}
I_1\le 0.8+0.1\delta.
\end{align*}
Regarding $I_2$, by triangle inequality we have 
\begin{align*}
I_2&\le \int_{\zb\in\cA} (1-\alpha_\ub(\zb))P(\zb|\ub)\ind(\zb\in\cS_\ub)\dd\zb+\int_{\zb\in\cA} (1-\alpha_\vb(\zb))P(\zb|\vb)\ind(\zb\in\cS_\vb)\dd\zb  \notag\\
&\qquad + \bigg|\int_{\zb\in\cA} p(\ub)P(\zb|\ub) - p(\vb)P(\zb|\vb)\dd \zb\bigg|\notag\\
&\le \delta + \underbrace{\bigg|\int_{\zb\in\cA} P(\zb|\ub)\ind(\zb\in\cS_\ub) - P(\zb|\vb)\ind(\zb\in\cS_\vb)\dd \zb\bigg|}_{I_3},
\end{align*}
where the last inequality is by \eqref{eq:lowerbound_accept2}.
Regarding $I_3$, we further have,
\begin{align*}
I_3&\le \bigg|\int_{\zb\in\cA} \ind(\zb\in\cS_\vb)\big(P(\zb|\ub)-P(\zb|\vb)\big)\dd \zb\bigg| + \bigg|\int_{\zb\in\cA} \big[\ind(\zb\in\cS_\ub)-\ind(\zb\in\cS_\vb)\big]P(\zb|\ub)\dd \zb\bigg|\notag\\
&\le \|P(\cdot|\ub)-P(\cdot|\vb)\|_{TV} + \max\bigg\{\int_{\zb\in\cS_{\vb}\bs\cS_{\ub}}P(\zb|\ub)\dd \zb,\int_{\zb\in\cS_{\ub}\bs\cS_{\vb}}P(\zb|\ub)\dd \zb \bigg\}\notag\\
&\le \|P(\cdot|\ub)-P(\cdot|\vb)\|_{TV} + \max\bigg\{\int_{\zb\in\RR^d\bs\cS_{\ub}}P(\zb|\ub)\dd \zb,\int_{\zb\in\RR^d\bs\cS_{\vb}}P(\zb|\ub)\dd \zb \bigg\}.
\end{align*}
For any $\cI$, note that $P(\zb|\ub,\cI)$ is a Gaussian distribution with mean $\ub - \eta \gb(\ub,\cI)$ and covariance matrix $2\eta\Ib/\beta$, thus we have
\begin{align*}
\int_{\RR^d\bs\cS_\ub} P(\zb|\ub,\cI)\dd\zb&\le\PP_{z\sim\chi_d^2}\big(z\ge 0.5\beta(r-\eta\|\gb(\ub,\cI)\|_2)^2 /\eta \big) \notag\\
\int_{\RR^d\bs\cS_\vb} P(\zb|\ub,\cI)\dd\zb&\le\PP_{z\sim\chi_d^2}\big(z\ge 0.5\beta(r-\eta\|\gb(\ub,\cI)\|_2 - \|\ub-\vb\|_2)^2 /\eta \big).
\end{align*}
Note that the above inequalities hold for any choice of $\cI$. Thus, if $\|\ub-\vb\|_2\le 0.1r$ and $\eta\le 0.1d\beta^{-1}/(LR+G)^2$, by Lemma \ref{lemma:grad_bound}, we have $r-\eta\|\gb(\ub,\cI)\|_2 - \|\ub-\vb\|_2\ge \sqrt{6.4\eta d/\beta}$ since $r\ge \sqrt{10\eta d/\beta}$, and then
\begin{align*}
\max\bigg\{\int_{\zb\in\RR^d\bs\cS_{\ub}}P(\zb|\ub)\dd \zb,\int_{\zb\in\RR^d\bs\cS_{\vb}}P(\zb|\ub)\dd \zb \bigg\}\le \PP_{z\sim\chi_d^2}\big(z\le 3.2d\big) \le 0.1.
\end{align*}
Then combining the above results and apply Lemma \ref{lemma:TV_expected_gaussian}, assume $\eta\le 1/L$, we have 
\begin{align*}
I_3\le 0.1 + \|P(\cdot|\ub)-P(\cdot|\vb)\|_{TV}\le 0.1 +\sqrt{2\beta}\|\ub - \vb\|_2/\sqrt{\eta}.
\end{align*}
This immediately implies that $I_2\le \delta + \sqrt{2\beta}\|\ub-\vb\|_2/\sqrt{\eta}+0.1$ and finally
\begin{align*}
\|\transitRef_\ub(\cdot) - \transitRef_\vb(\cdot)\|_{TV}\le I_1 + I_2/2 \le 0.85+0.1\delta + \frac{\sqrt{\beta}\|\ub-\vb\|_2}{\sqrt{2\eta}}.
\end{align*}
By Lemma \ref{lemma:approximation}, we know that if $\eta\le [25\beta(LR+G)^2]^{-1}$, we have
\begin{align*}
\delta&=\big[10Ld\eta  +10L(LR+G)d^{1/2}\beta^{1/2}\eta^{3/2} + 12\beta (LR+G)^2 d\eta/B+ 2\beta^2(LR+G)^4\eta^2/B\big]\\
&\qquad\cdot\bigg(1+\sqrt{\frac{\log(8K/\epsilon)}{d}}\bigg)^2\notag\\
&\le \big[14Ld\eta + 14(LR+G)^2\beta d\eta/B\big]\cdot\bigg(1+\sqrt{\frac{\log(8K/\epsilon)}{d}}\bigg)^2.
\end{align*}
Thus if 
\begin{align*}
\eta \le \frac{1}{25\beta(LR+G)^2}\wedge \frac{1}{35(Ld+(LR+G)^2\beta d/B)} \quad \mbox{and} \quad \|\ub-\vb\|_2\le \frac{\sqrt{2\eta}}{10\sqrt{\beta}}\le 0.1 r,
\end{align*}
we have $\|\transitRef_\ub(\cdot) - \transitRef_\vb(\cdot)\|_{TV}\le 0.99$. Then by Lemma \ref{lemma:lowerbound_transitionprob}, we have the following lower bound on the conductance of $\transitRef_\ub(\cdot)$
\begin{align*}
\phi \ge c_0\rho\sqrt{\eta/\beta},
\end{align*}
where $c_0$ is an absolute constant.
This completes the proof.
\end{proof}

\subsection{Proof of Lemma \ref{lemma:approximate_target}}
We present the following useful lemma that characterizes the probability measure of the region $\cB(\zero,R)$ under the  target distribution $\pi$.
\begin{lemma}\label{lemma:high_prob}
Under Assumptions \ref{assump:diss} and \ref{assump:smooth}, let $\Omega = \cB(\zero,\bar R(\zeta))$ for some $\zeta\in(0,1)$, it holds that
\begin{align*}
\pi(\Omega)\ge 1-\zeta.
\end{align*}
\end{lemma}
\begin{proof}[Proof of Lemma \ref{lemma:approximate_target}]
According to the definition of total variation distance, we know that there exists a set $\cA\in\RR^d$ such that
\begin{align*}
\|\stationalRef - \pi\|_{TV}  = |\stationalRef(\cA) - \pi(\cA)| &\le |\stationalRef(\cA\cap\Omega) - \pi(\cA\cap\Omega)| + \pi\big(\cA\bs\Omega\big)\notag\\
&\le |\stationalRef(\cA\cap\Omega) - \pi(\cA\cap\Omega)| + \pi\big(\RR^d\bs\Omega\big),
\end{align*}
where the first inequality is by triangle inequality.
By Lemma \ref{lemma:high_prob}, we have 
$\pi(\RR^d\bs\Omega)\le \zeta$. For the first term on the R.H.S. of the above inequality, we have
\begin{align}\label{eq:upperbound_TV2}
|\stationalRef(\cA\cap\Omega) - \pi(\cA\cap\Omega)| = \bigg|\int_{\cA\cap\Omega}\stationalRef(\dd\ub) -\int_{\cA\cap\Omega} \pi(\dd\ub) \bigg| . 
\end{align}
Recall the definition of the truncated distribution $\stationalRef$, for any $\ub\in\Omega$, we have
\begin{align*}
    \pi(\dd\ub)=\frac{e^{-\beta f(\ub)}\dd\ub}{\int_{\RR^d} e^{-\beta f(\xb)}\dd \xb},\quad\text{and } \stationalRef(\dd\ub)=\frac{e^{-\beta f(\ub)}\dd\ub}{\int_{\Omega} e^{-\beta f(\xb)}\dd \xb}=\frac{\pi(\dd\ub)}{\pi(\Omega)},
\end{align*}
which immediately implies
\begin{align*}
\stationalRef(\dd\ub) - \pi(\dd\ub) =\bigg(\frac{1}{\pi(\Omega)}-1\bigg)\pi(\dd\ub)\le \frac{\zeta}{1-\zeta}\pi(\dd\ub).
\end{align*}
Plugging this into \eqref{eq:upperbound_TV2} yields
\begin{align*}
|\stationalRef(\cA\cap\Omega) - \pi(\cA\cap\Omega)|\le \frac{\zeta}{1-\zeta}\int_{\Ab\cap\Omega}\pi(\ub)\dd\ub\le \frac{\zeta}{1-\zeta}.
\end{align*}
Combining the above results, we have for any $\zeta\le 1/2$ that
\begin{align*}
\|\stationalRef - \pi\|_{TV}\le \zeta +  \frac{\zeta}{1-\zeta}\le 3\zeta.
\end{align*}
Set $\zeta=\epsilon/12$ we are able to  complete the proof.
\end{proof}

\section{Proof of Lemmas in Appendices \ref{sec:proof_remaining} and \ref{sec:proof_keylemma}}\label{sec:proof_lemma}
\subsection{Proof of Lemma \ref{lemma:quadratic_lower_bound}}
\begin{proof}[Proof of Lemma \ref{lemma:quadratic_lower_bound}]
We will prove this for two cases: 1) $\|\xb\|_2\le \sqrt{2b/m}$ and 2) $\|\xb\|_2\ge \sqrt{2b/m}$. 
For the first case, it is evident that 
\begin{align*}
f(\xb)\ge f(\xb^*)\ge f(\xb^*) + \frac{m}{4}\|\xb\|_2^2 - \frac{b}{2}.
\end{align*}
where the last inequality is due to the fact that $\|\xb\|_2\le \sqrt{2b/m}$. 
For the second case,  based on Assumption \ref{assump:diss}, define $g(\xb) =f(\xb) - m\|\xb\|_2^2/4 $, it is clear that
\begin{align*}
\big\la\nabla g(\xb),\xb\big\ra = \la\nabla f(\xb),\xb\ra -\frac{m}{2}\|\xb\|_2^2\ge \frac{m}{2}\|\xb\|_2^2 -b.    
\end{align*}
Therefore, if $\|\xb\|_2\ge \sqrt{2b/m}$, we have $\la\nabla g(\xb),\xb\ra\ge 0$ and thus we have $\la\nabla g(\xb),\alpha\xb\ra\ge 0$ for any $\alpha\ge0$. Then, for any $\xb$ with $\|\xb\|_2> \sqrt{2b/m}$, let $\yb = \sqrt{2b/m}\xb/\|\xb\|_2$, we have 
\begin{align}\label{eq:lowerbound_gx}
g(\xb) &= g(\yb)+\int_{0}^1\la\nabla g(\yb+t(\xb - \yb)),\xb-\yb\ra\dd t\ge g(\yb),
\end{align}
where the inequality is due to the facts that $\|\yb+t(\xb - \yb)\|_2\ge \sqrt{2b/m}$ and $\yb+t(\xb - \yb) = \alpha(\xb-\yb)$ with $\alpha = t+\sqrt{2b/m}/(\|\xb\|_2 -\sqrt{2b/m} )$. By the definition of function $g(\cdot)$, we have that for any $\yb$ with $\|\yb\|_2\le \sqrt{2b/m}$,
\begin{align}\label{eq:lowerbound_gz}
g(\yb) = f(\yb) - m\|\yb\|_2^2/4 \ge f(\xb^*) - b/2.
\end{align}
Plugging \eqref{eq:lowerbound_gz} into \eqref{eq:lowerbound_gx} gives
\begin{align*}
g(\xb) \ge g(\yb) \ge f(\xb^*) - b/2.
\end{align*}
thus it follows that
\begin{align}\label{eq:lowerbound_func}
f(\xb)\ge \frac{m}{4}\|\xb\|_2^2 + f(\xb^*) - \frac{b}{2},
\end{align}
which completes the proof.
\end{proof}

\subsection{Proof of Lemma \ref{lemma:approximation_gld}}
\begin{proof}[Proof of Lemma \ref{lemma:approximation_gld}]
Similar to the proof of Lemma \ref{lemma:approximation}, the essential part is to prove that $\alpha_\ub(\wb)\ge 1-\delta/2$ for all $\wb\in\cB(\ub,r)\cap\Omega\bs\{\ub\}$. We will prove that under Assumption \ref{assump:hessian_lip}, the first term on the R.H.S. of \eqref{eq:lower_bound_I1/I2} can be improved.
By Assumption \ref{assump:hessian_lip} and \citet{nesterov2018lectures}, we know 
\begin{align*}
f(\wb) - f(\ub)&\le \la\wb-\ub,\nabla f(\ub)\ra + \frac{1}{2}(\wb - \ub)^\top\nabla^2 f(\ub) (\wb-\ub)+\frac{H}{6}\|\wb - \ub\|_2^3, \notag\\
f(\ub) - f(\wb)&\ge \la\ub-\wb,\nabla f(\wb)\ra + \frac{1}{2}(\ub - \wb)^\top\nabla^2 f(\wb) (\ub-\wb)-\frac{H}{6}\|\ub - \wb\|_2^3.
\end{align*}
Thus it follows that
\begin{align*}
&\big|2f(\wb) - 2f(\ub)  - \la\wb - \ub,\nabla f(\ub)+f(\wb)\ra\big|\notag\\
&\le \frac{1}{2}\big|(\wb - \ub)^\top\big(\nabla^2 f(\ub) -\nabla^2f(\wb) \big)(\wb-\ub)\big| + \frac{H}{3}\|\wb - \ub\|_2^3\notag\\
&\le \frac{5H}{6}\|\wb - \ub\|_2^3,
\end{align*}
where the second inequality is by Assumption \ref{assump:hessian_lip} as well. Then combining with \eqref{eq:upperbound_variance_difference}, let $I_1$ and $I_2$ be the same as those in the proof of Lemma \ref{lemma:approximation} and note that $\|\wb-\ub\|_2\le r = \sqrt{10\eta d/\beta}\big(1+\sqrt{\log(8K/\epsilon)/d}\big)$, we can derive the following by \eqref{eq:lower_bound_I1/I2},
\begin{align*}
\frac{I_1}{I_2}&\ge \exp\bigg(-\frac{5H\beta\|\wb - \ub\|_2^3}{12} -\frac{\beta \eta L(LR+G)\|\wb-\ub\|_2 }{2} \notag\\
&\hspace{30mm}-\frac{\beta^2 M^2\big(\|\wb-\ub\|_2^2+(LR+G)^2\eta^2+\beta^{-1}\eta d/2\big)}{2B}\bigg)\notag\\
&\ge \exp\bigg[-\bigg(14Hd^{3/2}\beta^{-1/2}\eta^{3/2} -5L(LR+G)d^{1/2}\beta^{1/2}\eta^{3/2} -\frac{6\beta M^2 d\eta}{B}- \frac{\beta^2M^2(LR+G)^2\eta^2}{2B}\bigg)\notag\\
&\hspace{20mm}\cdot\bigg(1+\sqrt{\frac{\log(8K/\epsilon)}{d}}\bigg)^2\bigg]\notag\\
&\ge 1 -\big[14Hd^{3/2}\beta^{-1/2}\eta^{3/2} -5L(LR+G)d^{1/2}\beta^{1/2}\eta^{3/2} - 6\beta (LR+G)^2 d\eta/B- \beta^2(LR+G)^4\eta^2/B\big]  \notag\\
&\qquad\cdot\bigg(1+\sqrt{\frac{\log(8K/\epsilon)}{d}}\bigg)^2\notag\\
&= 1- \delta/2,
\end{align*}
where we use the fact that $M=LR+G$ in the last inequality.
Then following the same procedure as in the proof of Lemma \ref{lemma:approximation}, we are able to complete the proof.
\end{proof}


\subsection{Proof of Lemma \ref{lemma:subgaussian_without}}
\begin{proof}[Proof of Lemma \ref{lemma:subgaussian_without}]
Note that we have $\|\xb\|_2\le R$, then by Lemma \ref{lemma:grad_bound} we know that
\begin{align*}
\|\gb(\xb,\cI_1)-\nabla f(\xb)\|_2 -\|\gb(\xb,\cI_2)-\nabla f(\xb)\|_2\le \|\gb(\xb,\cI_1)-\gb(\xb,\cI_2)\|_2 \le 2LR+2G,
\end{align*}
for all $\xb$. Then by Hoeffding's lemma, we have that there exists a constant $M=LR+G$ such that
\begin{align*}
\EE_{\cI}\big[\exp\big(\la \ab,\gb(\xb,\cI)-\nabla f(\xb)\ra\big)\big]\le \exp(M^2\|\ab\|_2^2)
\end{align*}
for any $\ab\in\RR^d$.
Moreover, note that $\cI$ is uniformly sampled from $[n]$ without replacement. Let $\cI'$ be the stochastic mini-batch sampled from $[n]$ with replacement, by Lemma 1.1 in \citet{bardenet2015concentration} and the convexity of function $\exp(\cdot)$, we have
\begin{align*}
\EE_{\cI}\big[\exp\big(\la \ab,\gb(\xb,\cI)-\nabla f(\xb)\ra\big)\big]\le \EE_{\cI'}\big[\exp\big(\la \ab,\gb(\xb,\cI')-\nabla f(\xb)\ra\big)\big].
\end{align*}
Then based on the fact that each element in $\cI'$ is independently drawn from $[n]$, we have
\begin{align*}
\EE_{\cI'}\big[\exp\big(\la \ab,\gb(\xb,\cI')-\nabla f(\xb)\ra\big)\big] &= \EE_{\cI'}\bigg[\prod_{i\in\cI'}\exp\bigg(\frac{1}{B}\la \ab,\gb(\xb,\{i\})-\nabla f(\xb)\ra\bigg)\bigg]\notag\\
& = \prod_{i\in\cI'}\EE_{i}\bigg[\exp\bigg(\frac{1}{B}\la \ab,\gb(\xb,\{i\})-\nabla f(\xb)\ra\bigg)\bigg]\notag\\
& \le \prod_{i\in\cI'}\exp\big(M^2\|\ab\|_2^2/B^2\big)\notag\\
& = \exp\big(M^2\|\ab\|_2^2/B\big).
\end{align*}
Furthermore, note that if $B = n$, we have $\EE_{\cI}\big[\exp\big(\la \ab,\gb(\xb,\cI)-\nabla f(\xb)\ra\big)\big] = 1$.
This completes the proof.
\end{proof}

\subsection{Proof of Lemma \ref{lemma:contraction}}
\begin{lemma}[Lemma 1.2 in \citet{lovasz1993random}]\label{lemma:temp111}
For any atom-free distributions $\mu$ and $\nu$ on $\Omega$, define function
\begin{align*}
l(p) = \sup_{g:\Omega\rightarrow[0,1]}\int_\Omega g(\xb)\mu(\dd\xb)\quad \text{s.t.}\quad  \int_\Omega g(\xb)\nu(\dd\xb) = p.
\end{align*}
Then there exists a set $\cA\in\Omega$ with $\nu(\cA) = p$ such that $\ell(p) = \mu(\cA)$.
\end{lemma}

\begin{proof}[Proof of Lemma \ref{lemma:contraction}]
Similar to the proof of Lemma 1.3 in \citet{lovasz1993random},
we first define the following functions for all $\ub,\cA\in\Omega$,
\begin{align*}
g_1(\ub,\cA) = \left\{
             \begin{array}{lr}
             2\transitSGLD_\ub(\cA) - 1, & \ub\in\cA\\
             0, & \ub\notin\cA\\
             \end{array}
\right. \qquad 
g_2(\ub,\cA) = \left\{
             \begin{array}{lr}
             1, & \ub\in\cA\\
             2\transitSGLD_\ub(\cA), & \ub\notin\cA\\
             \end{array}
\right.
\end{align*}
It is easy to see that $g_1(\ub,\cA)+g_2(\ub,\cA)=2\transitSGLD_{\ub}(\cA)$ for all $\ub\in\Omega$. In addition, for a $1/2$-lazy Markov process $\transitSGLD_\ub(\cdot)$ defined as in \eqref{eq:def_trans_lz_sgld}, we have $g_1(\cdot,\cdot),g_2(\cdot,\cdot)\in[0,1]$. Based on the above definitions, we can further derive that
\begin{align}\label{eq:stationry_measure1}
\int_{\Omega}g_1(\ub,\cA)\stationalRef(\dd\ub) &= \int_{\cA}[2\transitSGLD_\ub(\cA) -1] \stationalRef(\dd\ub) \notag\\
& = \int_{\cA}[1 - 2\transitSGLD_\ub(\Omega\bs\cA)]\stationalRef(\dd\ub)\notag\\
& = \stationalRef(\cA) - 2\int_{\cA}\transitRef_\ub(\Omega\bs\cA)\stationalRef(\dd\ub) - 2\underbrace{\int_{\cA}\big[\transitSGLD_\ub(\Omega\bs\cA) -\transitRef_\ub(\Omega\bs\cA)\big]\stationalRef(\dd\ub)}_{r_1},
\end{align}
and 
\begin{align}\label{eq:stationry_measure2}
\int_{\Omega} g_2(\ub,\cA)\stationalRef(\dd\ub) &= \stationalRef(\cA) + \int_{\Omega\bs\cA} 2\transitSGLD_\ub(\cA)\stationalRef(\dd\ub)\notag\\
& = \stationalRef(\cA) + 2\int_{\Omega\bs\cA}  \transitRef_\ub(\cA)\stationalRef(\dd\ub) + 2\int_{\Omega\bs\cA}  \big[\transitSGLD_\ub(\cA) - \transitRef_\ub(\cA)\big] \stationalRef(\dd\ub)\notag\\
& = \stationalRef(\cA) + 2\int_{\cA}  \transitRef_\ub(\Omega\bs\cA)\stationalRef(\dd\ub) +2 \underbrace{\int_{\Omega\bs\cA}  \big[\transitSGLD_\ub(\cA) - \transitRef_\ub(\cA)\big] \stationalRef(\dd\ub)}_{r_2},
\end{align}
where the last equality is by the fact that $\transitRef_\ub(\cdot)$ is a time-reversible Markov chain with stationary distribution $\stationalRef(\cdot)$.
Note that we have $(1-\delta)\transitRef_\ub(\cA)\le \transitSGLD_\ub(\cA)\le (1+\delta)\transitRef_\ub(\cA)$, the approximation error terms $r_1$ and $r_2$ can be upper bounded as follows,
\begin{align*}
|r_1|&\le \delta \int_{\cA}\transitRef_\ub(\Omega\bs\cA)\stationalRef(\dd\ub)\notag\\
|r_2|& \le \delta \int_{\Omega\bs\cA}\transitRef_\ub(\cA)\stationalRef(\dd\ub) = \delta \int_{\cA}\transitRef_\ub(\Omega\bs\cA)\stationalRef(\dd\ub).
\end{align*}
Then, combining \eqref{eq:stationry_measure1} and \eqref{eq:stationry_measure2} gives
\begin{align}\label{eq:stationary_measure_total}
\int_{\Omega} [g_1(\ub,\cA)+g_2(\ub,\cA)]\stationalRef(\dd\ub) = 2\stationalRef(\cA) -2r_1 + 2r_2.
\end{align}
Based on the definition of $h_k(p)$, we know that  there exists a set $\cA_k$ satisfying $\stationalRef(\cA_k) = p$ such that
\begin{align}\label{eq:def_hk}
h_k(p) = \mu_k(\cA_k) - \stationalRef(\cA_k).
\end{align}
Moreover, note that the distribution $\mu_k(\cdot)$ is generated by conducting one-step transition (based on transition distribution $\transitSGLD_\ub(\cdot)$) from distribution $\mu_{k-1}(\cdot)$, we have
\begin{align*}
\mu_k(\cA_k) = \int_{\cA_k} \mu_k(\dd\ub) = \int_{\Omega}\transitSGLD_\ub(\cA_k) \mu_{k-1}(\dd\ub).
\end{align*}
Based on the definitions of functions $g_1$ and $g_2$, the above equation can be reformulated as
\begin{align}\label{eq:split_h_k}
\mu_k(\cA_k) = \frac{1}{2}\int_{\Omega} \big[ g_1(\ub,\cA_k) + g_2(\ub,\cA_k)\big] \mu_{k-1}(\dd\ub).
\end{align}
By \eqref{eq:stationary_measure_total}, we know that
\begin{align}\label{eq:stationary_measure_k}
\int_{\Omega} [g_1(\ub,\cA_k)+g_2(\ub,\cA_k)]\stationalRef(\dd\ub) = 2\stationalRef(\cA_k) -2r_1 + 2r_2,
\end{align}
where $r_1$ and $r_2$ are two approximation error terms satisfying
\begin{align}\label{eq:bound_r1r2}
|r_1|, |r_2|&\le \delta \int_{\cA_k}\transitRef_\ub(\Omega\bs\cA_k)\stationalRef(\dd\ub).
\end{align}
Then based on Lemma \ref{lemma:temp111}, we know that there exist two sets $\cA_{k-1}^1,\cA_{k-1}^2\subseteq\Omega$ satisfying
\begin{align}\label{eq:def_p1p2}
p_1:=\stationalRef(\cA_{k-1}^1) = \int_\Omega g_1(\ub,\cA_k) \stationalRef(\dd \ub) \quad \mbox{and}\quad p_2:=\stationalRef(\cA_{k-1}^2) = \int_\Omega g_2(\ub,\cA_k) \stationalRef(\dd \ub),
\end{align}
such that 
\begin{align*}
\int_{\Omega} g_1(\ub,\cA_k)\mu_{k-1}(\dd\ub)\le \mu_{k-1}(\cA_{k-1}^1) \quad \mbox{and}\quad \int_{\Omega} g_2(\ub,\cA_k)\mu_{k-1}(\dd\ub)\le \mu_{k-1}(\cA_{k-1}^2).
\end{align*}
Therefore, based on \eqref{eq:split_h_k} and \eqref{eq:stationary_measure_k}, we have
\begin{align*}
h_k(p) &= \mu_k(\cA_k) - \stationalRef(\cA_k)\notag\\
& = \frac{1}{2}\int_{\Omega} \big[ g_1(\ub,\cA_k) + g_2(\ub,\cA_k)\big] \mu_{k-1}(\dd\ub) - \frac{1}{2}\int_{\Omega} \big[ g_1(\ub,\cA_k) + g_2(\ub,\cA_k)\big] \pi^*(\dd\ub) \notag\\
&\qquad-r_1+r_2\\
&\le \frac{1}{2}\big[\mu_{k-1}(\cA_{k-1}^1)+\mu_{k-1}(\cA_{k-1}^2) - \stationalRef(\cA_{k-1}^1) - \stationalRef(\cA_{k-1}^2)\big]+|r_1-r_2|\notag\\
& \le \frac{1}{2}\big[h_{k-1}(p_1) + h_{k-1}(p_2)\big]+ |r_1-r_2|,
\end{align*}
where the first inequality is by \eqref{eq:stationary_measure_k} and \eqref{eq:def_p1p2} and triangle inequality, and the last equality is by the definition of function $h_{k-1}(\cdot)$. Recalling \eqref{eq:stationry_measure1}, \eqref{eq:stationry_measure2}, and \eqref{eq:def_p1p2} the probabilities $p_1$ and $p_2$ can be reformulate as
\begin{align*}
p_1 = p - \tilde\phi_k\min\{p, 1-p\}\quad \mbox{and}\quad p_2 = p + \hat\phi_k\min\{p, 1-p\},
\end{align*}
where
\begin{align*}
\tilde\phi_k = \frac{2\int_{\cA_k}\transitRef_\ub(\Omega\bs\cA_k)\stationalRef(\dd\ub)-2r_1}{\min\{p, 1-p\}}\quad \mbox{and}\quad \hat\phi_k = \frac{2\int_{\cA_k}\transitRef_\ub(\Omega\bs\cA_k)\stationalRef(\dd\ub)+2r_2}{\min\{p, 1-p\}}.
\end{align*}
We further define
\begin{align*}
\phi_k=\frac{ \int_{\cA_k}\transitRef_\ub(\Omega\bs\cA_k)\stationalRef(\dd \ub)}{\min\{p, 1-p\}}.
\end{align*}
Apparently, according to Definition \ref{def:s-conductance}, it holds that $\phi_k\ge \phi_s$. In addition, by \eqref{eq:bound_r1r2} and our definitions of $\tilde \phi_k$ and $\hat \phi_k$, it can be also derived that 
\begin{align*}
2(1-\delta)\phi_k\le\tilde\phi_k\leq\hat\phi_k\le 2(1+\delta)\phi_k.
\end{align*}
Since the transition kernel $\transitSGLD_\ub(\cdot)$ is $1/2$-lazy, we have $\tilde \phi_k\le 1$. Then, if $\delta\le \min\{1-\sqrt{2}/2,\phi/16\}$, we have $\tilde\phi_k\ge \sqrt{2}\phi_k$ and $\hat\phi_k - \tilde\phi_k\le 4\delta\phi_k\le \tilde\phi_k^2/4$. Moreover, note that $\sqrt{1-x}\le 1 - x/2 - x^2/8$ for all $x\in(0,1)$, we have
\begin{align*}
\sqrt{1-\tilde\phi_k} + \sqrt{1+\hat\phi_k} & = \sqrt{1-\tilde\phi_k} + \sqrt{1+\tilde\phi_k + \hat\phi_k - \tilde\phi_k}\notag\\
& \le 1 - \frac{\tilde\phi_k}{2} - \frac{\tilde\phi_k^2}{8} + 1 + \frac{\tilde\phi_k}{2}\notag\\
&\le 2 - \frac{\phi_k^2}{4}.
\end{align*}
Moreover, \eqref{eq:bound_r1r2} also implies that
\begin{align*}
|r_1-r_2|\le 2\delta\int_{\cA_k} \transitRef_\ub(\Omega\bs\cA_k)\stationalRef(\dd\ub) = 2\delta\phi_k \min\{p, 1-p\}\le 2\delta\phi_k \min\big\{\sqrt{p}, \sqrt{1-p}\big\}.
\end{align*}
Therefore, we can finally upper bound $h_k(p)$ as follows,
\begin{align*}
h_k(p)
&\le  \frac{1}{2}\big[h_{k-1}\big(p-\tilde\phi_k\min\{p, 1-p\}\big) + h_{k-1}\big(p+ \hat\phi_k\min\{p, 1-p\}\big)\big] \notag\\
&+ 2\delta\phi_k\min\big\{\sqrt{p}, \sqrt{1-p}\big\},
\end{align*}
which completes the proof.
\end{proof}

\subsection{Proof of Lemma \ref{lemma:lowerboud_accept_prob_alg}}
\begin{proof}[Proof of Lemma \ref{lemma:lowerboud_accept_prob_alg}]
For any $\cI$, we have
\begin{align}\label{eq:lower_bound_prob}
\int_{\Omega\cap \cB(\ub,r)} P(\zb|\ub,\cI)\dd \zb  &= \int_{\cB(\ub,r)} P(\zb|\ub,\cI) \dd \zb - \int_{\cB(\ub,r)\bs\Omega} P(\zb|\ub,\cI) \dd \zb\notag\\
&\ge\underbrace{ \int_{\cB(\ub,r)} P(\zb|\ub,\cI) \dd \zb}_{I_1} - \underbrace{\int_{ \RR^d\bs\Omega} P(\zb|\ub,\cI) \dd \zb}_{I_2}.
\end{align}
Thus the remaining part is to prove that the R.H.S. of the above inequality is greater than 0.4.
Regarding $I_1$, note that $P(\zb|\ub,\cI)$ is a Gaussian distribution with mean $\ub - \eta \gb(\ub,\cI)$ and covariance matrix $2\eta\Ib/\beta$, thus we have
\begin{align*}
\int_{\cB(\ub,r)} P(\zb|\ub,\cI) \dd \zb \ge \PP_{z\sim\chi_d^2}\big(z\le 0.5\beta(r-\eta\|\gb(\ub,\cI)\|_2)^2/\eta \big).
\end{align*}
By Lemma \ref{lemma:grad_bound}, we know that $\|\gb(\ub,\cI)\|_2\le LR+G$. Therefore, based on our choice $r = \sqrt{10 \eta d/\beta}$, it is clear that if 
\begin{align*}
\eta \le \frac{0.1d}{\beta(LR+G)^2},
\end{align*}
we have $r-\eta\|\gb(\ub,\cI)\|_2\ge \sqrt{8\eta d/\beta}$ and thus
\begin{align}\label{eq:bound_prob1}
\int_{\cB(\ub,r)} P(\zb|\ub,\cI) \dd \zb \ge \PP_{z\sim\chi_d^2}(z\le 4d)\ge 0.95.
\end{align}
Then we will prove the upper bound of $I_2$. Note that the set $\Omega$ is a ball centered at the origin and $\ub\in\Omega$, we can construct a point $\wb$ as follows,
\begin{align*}
\wb = \ub - \frac{\big(R - \sqrt{R^2-r^2}\big)\ub}{\|\ub\|_2}.
\end{align*}
It is easy to see that a half space of $\cB(\wb,r)$ is contained by the set $\Omega$. Let $Q(\zb|\wb) = N(\wb, 2\eta\Ib/\beta)$, it follows that 
\begin{align*}
\int_{\Omega}Q(\zb|\wb)\dd\zb \ge \int_{\Omega\cap \cB(\wb,r)}Q(\zb|\wb)\dd\zb \ge \frac{1}{2}\int_{ \cB(\wb,r)}Q(\zb|\wb)\dd\zb.
\end{align*}
Note that $Q(\zb|\wb)$ is a Gaussian distribution with mean $\wb$ and covariance matrix $2\eta\Ib/\beta$, thus we have 
\begin{align*}
\int_{\Omega}Q(\zb|\wb)\dd\zb\ge\frac{1}{2} \PP_{z\in\chi_d^2}(z\le 5d)\ge 0.475.
\end{align*}
Moreover, by Pinsker's  inequality \citep{cover2012elements}, we have
\begin{align*}
\bigg|\int_{\Omega}P(\zb|\ub,\cI)\dd\zb -  \int_{\Omega}Q(\zb|\wb)\dd\zb\bigg|\le \|P(\cdot|\ub,\cI) - Q(\cdot|\wb)\|_{TV}\le \sqrt{2D_{KL}\big(P(\cdot|\ub,\cI), Q(\cdot|\wb)\big)}.
\end{align*}
Note that $P(\cdot|\ub,\cI)$ and  $Q(\cdot|\wb)$ are Gaussian distributions with the same covariance matrices, we have $D_{KL}\big(P(\cdot|\ub,\cI), Q(\cdot|\wb)\big) = \beta\|\ub-\vb\|_2^2/(4\eta)$. Therefore, it follows that
\begin{align*}
\bigg|\int_{\Omega}P(\zb|\ub,\cI)\dd\zb -  \int_{\Omega}Q(\zb|\wb)\dd\zb\bigg|\le \frac{\|\wb - (\ub - \eta \gb(\ub,\cI))\|_2}{\sqrt{2\eta/\beta}}\le \sqrt{\beta/2\eta} \cdot\big(\|\wb - \ub\|_2 + \eta\|\gb(\ub,\cI)\|_2\big).
\end{align*}
By our construction of $\wb$, we have
\begin{align*}
\|\wb - \ub\|_2 = R - \sqrt{R^2 - r^2} = R\big(1 - \sqrt{1-r^2/R^2}\big) \le r^2/R = \frac{10\eta d}{\beta R}]\cdot\bigg(1+\sqrt{\frac{\log(8K/\epsilon)}{d}}\bigg). 
\end{align*}
By Lemma \ref{lemma:grad_bound}, we know that $\|\gb(\ub,\cI)\|_2\le LR+G$.
Therefore, if the step size satisfies
\begin{align*}
\eta \le \frac{1}{40\big(1+\sqrt{\log(8K/\epsilon)/d}\big)}\big(d\beta^{-1}R^{-1}+(LR+G)\beta^{1/2}\big)^{-1},
\end{align*}
we have $\|P(\cdot|\ub,\cI) - Q(\cdot|\wb)\|_{TV}\le 0.025$, and thus
\begin{align*}
\int_{\RR^d\bs\Omega}P(\zb|\ub,\cI)\dd\zb = 1-\int_{\Omega}P(\zb|\ub,\cI)\dd\zb \le 1- \int_{\Omega}Q(\zb|\wb)\dd\zb+0.025 = 0.55.
\end{align*}
Combining with \eqref{eq:bound_prob1}, we have the following by \eqref{eq:lower_bound_prob},
\begin{align*}
\int_{\Omega\cap \cB(\ub,r)} P(\zb|\ub,\cI)\dd \zb  
&\ge \int_{\cB(\ub,r)} P(\zb|\ub,\cI) \dd \zb - \int_{ \RR^d\bs\Omega} P(\zb|\ub,\cI) \dd \zb
\ge 0.95 - 0.55=0.4.
\end{align*}
This completes the proof.

\end{proof}

\subsection{Proof of Lemma \ref{lemma:TV_expected_gaussian}}
\begin{proof}[Proof of Lemma \ref{lemma:TV_expected_gaussian}]
By the definition of total variation distance, we know there exists a set $\cA\in\RR^d$ such that
\begin{align*}
\|P(\cdot|\ub) - P(\cdot|\vb)\|_{TV} &= |P(\cA|\ub) - P(\cA|\vb)| \notag\\
&= \bigg|\int_{\cA} P(\zb|\ub) - P(\zb|\vb)\dd \zb\bigg|\notag\\
&= \bigg|\EE_\cI\bigg[\int_{\cA} P(\zb|\ub,\cI) - P(\zb|\vb,\cI)\dd \zb\bigg]\bigg|\notag\\
&\le \EE_{\cI}[\|P(\zb|\ub,\cI) - P(\zb|\vb,\cI)\|_{TV}],
\end{align*}
where the last inequality is by triangle inequality and the definition of total variation distance. By Pinsker's inequality, we have
\begin{align*}
\|P(\zb|\ub,\cI) - P(\zb|\vb,\cI)\|_{TV} \le \sqrt{2D_{KL}\big(P(\cdot|\ub,\cI), P(\cdot|\vb,\cI)\big)} = \frac{\|\ub - \eta \gb(\ub,\cI) - (\vb - \eta \gb(\vb,\cI))\|_2}{\sqrt{2\eta/\beta}},
\end{align*}
where the last equality follows from the fact that $P(\cdot|\ub,\cI)$ and $P(\cdot|\vb, \cI)$ are two Gaussian distributions with different means and same covariance matrices. 
By triangle inequality, we have
\begin{align*}
\|\ub - \eta \gb(\ub,\cI) - (\vb - \eta \gb(\ub,\cI))\|_2\le \|\ub - \vb\|_2 + \eta \|\gb(\ub,\cI) - \gb(\vb,\cI)\|_2\le (1+L\eta)\|\ub-\vb\|_2,
\end{align*}
where the second inequality is by Assumption \ref{assump:smooth}.
Therefore, we have 
\begin{align*}
\|P(\cdot|\ub) - P(\cdot|\vb)\|_{TV}\le \EE_{\cI}[\|P(\zb|\ub,\cI) - P(\zb|\vb,\cI)\|_{TV}] \le \frac{(1+L\eta)\|\ub-\vb\|_2}{\sqrt{2\eta/\beta}}.
\end{align*}
This completes the proof.
\end{proof}

\subsection{Proof of Lemma \ref{lemma:high_prob}}
\begin{proof}[Proof of Lemma \ref{lemma:high_prob}]
Define a Gaussian distribution $q(\xb) = e^{-m\beta\| \xb\|_2^2/8}/[8\pi /(m\beta)]^{d/2}$. Then by \eqref{eq:lowerbound_function_quadratic} and proof of Corollary \ref{coro:weak_converge}, we have $q(\xb)\ge \pi(\xb)$ if $\|\xb\|_2^2\ge4m^{-1}(\beta^{-1}d\log(4L/m)+b)$. Thus, for any $\alpha \ge 4m^{-1}(\beta^{-1}d\log(4L/m)+b)$, we have
\begin{align*}
\int_{\|\xb\|_2^2\ge\alpha}\pi(\dd\xb)&\le \int_{\|\xb\|_2^2\ge\alpha}q(\xb)\dd \xb = \PP_{z\sim\chi_d^2}[z\ge m\beta\alpha/4],
\end{align*}
where the last equality is due to the fact that $q(\xb)$ is a Gaussian distribution with mean $\zero$ and covariance matrix $4\Ib/(m\beta)$. By standard tail bound of Chi-Square distribution, for any $\delta\in(0,1)$ we have
\begin{align*}
\PP_{z\sim\chi_d^2}\big[z\ge d + 2 \sqrt{d\log(1/\delta)}+2\log(1/\delta)\big]\le \delta.
\end{align*}
Therefore, define by $\Omega = \cB(\zero, \bar R(\zeta))$ with 
\begin{align*}
\bar R(\zeta) = \bigg[\max\bigg\{\frac{4d\log(4L/m)+4\beta b}{m\beta}, \frac{4d + 8 \sqrt{d\log(1/\delta)}+8\log(1/\delta)}{m\beta}\bigg\}\bigg]^{1/2},
\end{align*}
we have 
\begin{align*}
\pi(\Omega)  &= \int_{\|\xb\|_2^2\le \bar R(\zeta)}\pi(\dd\xb) \notag\\
& = 1- \int_{\|\xb\|_2^2\ge \bar R(\zeta)}\pi(\dd\xb) \notag\\
& \ge 1 - \PP_{z\sim\chi_d^2}[z\ge m\beta \bar R(\zeta)/4]\notag\\
&\ge 1- \zeta,
\end{align*}
which completes the proof.
\end{proof}

\bibliographystyle{ims}
\bibliography{MCMC}

\end{document}